\documentclass{article}

\usepackage{microtype}
\usepackage{graphicx}
\usepackage{subfigure}
\usepackage{booktabs} 
\usepackage{array}
\usepackage{amsthm}
\usepackage{amsmath}
\usepackage{mathtools}
\usepackage{amssymb}
\usepackage{comment}
\usepackage{tikz}
\usetikzlibrary{arrows}
\usepackage{dsfont}

\usepackage{verbatim}
\usepackage[colorinlistoftodos]{todonotes}

\DeclareMathOperator*{\argmax}{arg\,max}

\DeclareMathOperator*{\argmin}{arg\,min}

\DeclareMathOperator{\action}{\mathbf{a}}
\DeclareMathOperator{\baction}{\mathbf{b}}
\DeclareMathOperator{\state}{\mathbf{s}}
\DeclareMathOperator{\E}{\mathbb{E}} %

\usepackage{thm-restate}

\newtheorem{proposition}{Proposition}
\newtheorem{assumption}{Assumption}

\newtheorem{lemma}[proposition]{Lemma}
\newtheorem{theorem}[proposition]{Theorem}

\newtheorem{defn}{Definition}
\newtheorem{remark}{Remark}

\def\E{{\mathbb{E}}}
\def\eqdef{\stackrel{\text{def}}{=}}
\def\optPolicyPR{\pi^*_{P,\alpha}}
\def\optPolicyNR{\pi^*_{N,\alpha}}

\def\ValuePR{v_{P,\alpha}}
\def\ValueNR{v_{N,\alpha}}

\definecolor{yellowY}{HTML}{00F9DE}

\usepackage{hyperref}

\usepackage[accepted]{icml2019}

\icmltitlerunning{Action Robust Reinforcement Learning and Applications in Continuous Control}

\usepackage{bibunits}

\defaultbibliography{action_robust_bib}
\defaultbibliographystyle{icml2019}

\begin{document}

\twocolumn[
\icmltitle{Action Robust Reinforcement Learning and Applications in Continuous Control}



\icmlsetsymbol{equal}{*}

\begin{icmlauthorlist}
\icmlauthor{Chen Tessler}{equal,to}
\icmlauthor{Yonathan Efroni}{equal,to}
\icmlauthor{Shie Mannor}{to}
\end{icmlauthorlist}

\icmlaffiliation{to}{Department of Electrical Engineering, Technion Institute of Technology, Haifa, Israel}

\icmlcorrespondingauthor{Chen Tessler}{chen.tessler@campus.technion.ac.il}
\icmlcorrespondingauthor{Yonathan Efroni}{jonathan.e@campus.technion.ac.il}

\icmlkeywords{Machine Learning, ICML}

\vskip 0.3in
]




\printAffiliationsAndNotice{\icmlEqualContribution} 


\begin{abstract}
A policy is said to be robust if it maximizes the reward while considering a bad, or even adversarial, model. In this work we formalize two new criteria of robustness to action uncertainty. 
Specifically, we consider two scenarios in which the agent attempts to perform an action $\action$, and (i) with probability $\alpha$, an alternative adversarial action $\bar \action$ is taken, or (ii) an adversary adds a perturbation to the selected action in the case of continuous action space. We show that our criteria are related to common forms of uncertainty in robotics domains, such as the occurrence of abrupt forces, and suggest algorithms in the tabular case. Building on the suggested algorithms, we generalize our approach to deep reinforcement learning (DRL) and provide extensive experiments in the various MuJoCo domains. 
Our experiments show that not only does our approach produce robust policies, but it also improves the performance in the absence of perturbations. 
This generalization indicates that action-robustness can be thought of as implicit regularization in RL problems.

\end{abstract}

\section{Introduction}
Recent advances in Reinforcement Learning (RL) have demonstrated its potential in real-world deployment. However, since in RL it is normally assumed that the train and test domains are identical, it is not clear how a learned policy would generalize under small perturbations. For example, consider the task of robotic manipulation in which the task is to navigate towards a goal. As the policy is trained on a specific parameter set (mass, friction, etc...), it is not clear what would happen when these parameters change, e.g., if the robot is slightly lighter/heavier. 

The advantage of robust policies is highlighted when considering imperfect models, a common scenario in real world tasks such as autonomous vehicles. Even if the model is trained in the real world, certain variables such as traction, tire pressure, humidity, vehicle mass and road conditions may vary over time. These changes affect the dynamics of our model, a property which should be considered during the optimization process. Robust MDPs \cite{nilim2005robust,iyengar2005robust,wiesemann2013robust} tackle this issue by solving a max-min optimization problem over a set of possible model parameters, an uncertainty set, e.g., the range of values which the vehicle's mass may take - the goal is thus to maximize the reward, with respect to (w.r.t.) the worst possible outcome.
Previously, Robust MDPs have been analyzed extensively in the theoretical community, in the tabular case \cite{nilim2005robust,iyengar2005robust,xu2007robustness,mannor2012lightning,wiesemann2013robust} and under linear function approximation \cite{tamar2013scaling}. However, as these works analyze uncertainty in the transition probabilities: (i) it is not clear how to obtain these uncertainty sets, and (ii) it is not clear if and how these approaches may be extended to non-linear function approximation schemes, e.g., neural networks. Recently, this problem has been tackled, empirically, by the Deep RL community \cite{pinto2017robust,peng2018sim}. While these approaches seem to work well in practice, they require access and control of a simulator and are not backed by theoretical guarantees - a well known problem in adversarial training \cite{barnett2018convergence}.


Our approach tackles these problems by introducing a natural way to define robustness - robustness w.r.t. action perturbations - a scenario in which the agent attempts to perform an action and due to disturbances, such as noise or model uncertainty, acts differently than expected. In this work, we consider two distinct robustness criteria: given an action provided by the policy (i) the \textit{Probabilistic Action Robust MDP (PR-MDP, Section \ref{sec: probabilistic AR})} criterion considers the case in which, with probability $\alpha$, a different possibly adversarial action is taken; and (ii) the \textit{Noisy Action Robust MDP (NR-MDP, Section \ref{sec: noisy AR})} criterion, in which a perturbation is added to the action itself. These two criteria are strongly correlated to real world uncertainty; the former correlates to abrupt interruptions such as a sudden push and the latter correlates to a constant interrupting force. For instance, if the robot is heavier, this may be seen as an adversary applying force in the opposite direction \cite{bacsar2008h}.

In Section \ref{sec: experiments}, we extend our approach to Deep RL, perform extensive evaluation across several MuJoCo \cite{todorov2012mujoco} environments and show the ability of our approach to produce robust policies. We empirically analyze the differences between the PR-MDP and NR-MDP approaches, and demonstrate their ability to produce robust policies under abrupt perturbations and mass uncertainty. Surprisingly, we observe that even in the absence of perturbations, solving for the action robust criteria results in improved performance\footnote{Our code can be found in the following repository: \href{https://github.com/tesslerc/ActionRobustRL}{https://github.com/tesslerc/ActionRobustRL}}.


\section{Preliminaries}

\subsection{Markov Decision Process}\label{sec:mdp}
We consider the framework of infinite-horizon discounted Markov Decision Process (MDP) with continuous action space. An MDP is defined as the 5-tuple $(\mathcal{S}, \mathcal{A},P,R,\gamma)$ \cite{puterman1994markov}, where ${\mathcal S}$ is a finite state space, ${\mathcal A}$ is a compact and convex action metric space. We assume $P \equiv P(\state'|\state,\action)$ is a transition kernel and is weakly continuous in $\action$, $R \equiv r(\state,\action)$ is a reward function continuous in $a$, and $\gamma\in(0,1)$. Let $\pi: \mathcal{S}\rightarrow \mathcal{P}(\mathcal{A})$ be a stationary policy, where $\mathcal{P}(\mathcal{A})$ is the set of probability measures on the Borel sets of $\mathcal{A}$. We denote $\Pi$ as the set of stationary deterministic policies on $\mathcal{A}$, i.e., if $\pi\in\Pi$ then $\pi: \mathcal{S}\rightarrow \mathcal{A}$, and $\mathcal{P}(\Pi)$ as the set of stationary stochastic policies.  Let $v^\pi \in \mathbb{R}^{|\mathcal{S}|}$ be the value of a policy $\pi,$ defined in state $\state$ as $v^\pi(\state) \equiv \E^\pi[\sum_{t=0}^\infty\gamma^tr(\state_t,\action_t)\mid \state_0=\state]$, where $\action_t\sim \pi(\state_t)$ is a random-variable, $\E^\pi$ denotes expectation w.r.t. the distribution induced by $\pi$ and conditioned on the event $\{\state_0=\state\}.$ 


The goal is to find a policy $\pi^*,$ yielding the optimal value $v^*$, i.e., for all $\state\in \mathcal{S}$, ${\pi^*(\state) \in \argmax_{\pi'\in \mathcal{P}(\Pi)} \E^{\pi'}[\sum_{t=0}^\infty\gamma^tr(\state_t,\action_t)\mid \state_0=\state],}$  
and the optimal value is $v^{*}(\state) = v^{\pi^*}(\state)$. It is known, and quite surprising, that there always exists an optimal policy which is stationary and deterministic, meaning $\pi^*\in \Pi$, e.g., \cite{puterman1994markov}[Theorem 6.2.10].

We note that in all following results we assume continuity of the dynamics and reward in actions. For the exact definitions see Appendix~\ref{supp sec: pre MG}
, Assumption~\ref{assumptions: MG}.


\subsection{Zero-Sum Games}
As opposed to the standard MDP framework, in a two player zero-sum game, the reward function and transition kernels are functions of both players $\action\in \mathcal{A}$ and $\bar \action\in \bar{\mathcal{A}}$, where $\mathcal{A},\bar{\mathcal{A}}$ are compact sets. Assuming the policy of player 1 is $\pi$ and $\bar \pi$ of player 2, the value of the game is defined $\forall s\in \mathcal{S},\  v^{\pi,\bar{\pi}}(\state) \equiv \E^{\pi, \bar \pi}[\sum_{t=0}^\infty\gamma^tr(s_t,a_t,\bar a_t)\mid s_0=s]$. \citet{maitra1970stochastic} generalized result of \citet{shapley1953stochastic} and established that, under proper conditions, the zero sum game has value for any $s \in \mathcal{S}$, i.e.,
\begin{align*}
    v^{*} (\state) &= \max_{\pi \in  \mathcal{P}(\Pi)} \min_{\bar \pi \in \Pi} \E^{\pi, \bar \pi}[\sum_{t=0}^\infty\gamma^tr(s_t,a_t,\bar a_t)\mid s_0=s], \\
    &= \min_{\bar \pi \in \mathcal{P}(\Pi)} \max_{\pi \in \Pi} \E^{\pi, \bar \pi}[\sum_{t=0}^\infty\gamma^tr(s_t,a_t,\bar a_t)\mid s_0=s].
\end{align*}
Note that, in the general case, the optimal maximizing policy is selected from the set of stochastic policies. Policies which attain this value, $\pi^*$ and $\bar{\pi}^*$ for the maximizer and minimizer players, respectively, are said to be in Nash-Equilibrium. In such a scenario, neither player may improve it's outcome further, e.g., $\forall \pi, \bar\pi \in \mathcal{P}(\Pi)$, ${v^{\pi, \bar \pi^*} \leq v^{*} \leq v^{\pi^*, \bar\pi}}.$


\section{Probabilistic Action Robust MDP} \label{sec: probabilistic AR}
In this section we introduce the Probabilistic Action Robust MDP (PR-MDP), which can be viewed as a zero-sum game between an agent and an adversary. We refer to the optimal policy of the max-agent in PR-MDP as the optimal probabilistic robust policy. Furthermore, we establish that the game has a well defined value and analyze some properties of this criterion. Lastly, we formulate Policy Iteration (PI) schemes that solve the PR-MDP, and show that they inherit properties corresponding to single agent PI schemes.

\begin{restatable}{defn}{defPRMDP}
    \label{def: PR-MDP}Let $\alpha\in [0,1]$. A Probabilistic Action Robust MDP is defined by the 5-tuple of an MDP (see Section \ref{sec:mdp}). Let $\pi,\bar{\pi}$ be policies of an agent an adversary. We define their probabilistic joint policy $\pi_{P,\alpha}^{\mathrm{mix}}(\pi,\bar{\pi})$ as $\forall s \in \mathcal{S},\ \pi_{P,\alpha}^{\mathrm{mix}}(\action \mid \state) \equiv(1-\alpha)\pi(\action \mid \state)+\alpha \bar{\pi}(\action \mid \state).$
    
    Let $\pi$ be an agent policy. As opposed to standard MDPs, the value of the policy is defined by
    ${\ValuePR^\pi = \min_{\bar{\pi}\in \Pi} \mathbb{E}^{\pi_{P,\alpha}^{\mathrm{mix}}(\pi,\bar{\pi})}[\sum_t \gamma^t r(\state_t,\action_t)]},
    $ where ${\action_t \sim \pi_{P,\alpha}^{\mathrm{mix}}(\pi(\state_t),\bar{\pi}(\state_t))}$. The optimal probabilistic robust policy is the optimal policy of the PR-MDP
    \begin{align}
    \optPolicyPR \in \argmax_{\pi\in \mathcal{P}(\Pi)} \min_{\bar{\pi} \in \Pi} \mathbb{E}^{\pi_{P,\alpha}^{\mathrm{mix}}(\pi,\bar{\pi})}[\sum_t \gamma^t r(\state_t,\action_t)].\label{eq:PR-MDP}
    \end{align}
    The optimal probabilistic robust value is $\ValuePR^*=\ValuePR^{\optPolicyPR}$.
\end{restatable}

Simply put, an optimal probabilistic robust policy is optimal w.r.t. a scenario in which, with probability $\alpha$, an adversary takes control and performs the worst possible action. This approach formalizes a possible inability to control the system and to perform the wanted actions.

In-order to obtain the optimal probabilistic robust policy, one needs to solve the zero-sum game as defined in \eqref{eq:PR-MDP} (see Appendix~\ref{supp: PR-MDP and MG} 
for a formal mapping). It is well known \cite{straffin1993game} that any zero-sum game has a well defined value on the set of stochastic policies, but not always on the set of deterministic policies. Interestingly, and similarly to regular MDPs, the optimal policy of the PR-MDP is a deterministic one as the following proposition asserts (see proof in Appendix~\ref{supp: proof PR-MDP and determinstic optimal policy}).

\begin{proposition}\label{prop:PR-MDP and determinstic optimal policy} For PR-MDP, there exists an optimal policy which is stationary and deterministic, and strong duality holds in $\Pi$,
\begin{align*}
\ValuePR^* &= \max_{\pi\in \Pi} \min_{\bar{\pi} \in \Pi} \mathbb{E}^{\pi_{P,\alpha}^{\mathrm{mix}}(\pi,\bar{\pi})}[\sum_t \gamma^t r(\state_t,\action_t)]\\
& = \min_{\bar{\pi} \in \Pi} \max_{\pi\in \Pi} \mathbb{E}^{\pi_{P,\alpha}^{\mathrm{mix}}(\pi,\bar{\pi})}[\sum_t \gamma^t r(\state_t,\action_t)].
\end{align*}
\end{proposition}

\subsection{Probabilistic Action Robust and Robust MDPs}\label{sec: PR MDP and RMDPs}
Although the approach of PR-MDP might seem orthogonal to the that of Robust MDPs, the former is a specific case of the latter. By using the PR-MDP criterion, a class of models is implicitly defined, and the probabilistic robust policy is optimal w.r.t. the worst possible model in this class.  

To see the equivalence, define the following class of models,
\begin{align*}
&\mathcal{P}_\alpha = \{(1-\alpha)P+\alpha P^{\pi}:  \mathcal{P}(\Pi) \}\\
&\mathcal{R}_\alpha = \{(1-\alpha)r+\alpha r^{\pi}: \pi\in \mathcal{P}(\Pi) \}.
\end{align*} 
A probabilistic robust policy, which solves \eqref{eq:PR-MDP}, is also the solution to the following RMDP (see Appendix~\ref{supp: PR-MDP and RMDP}),
\begin{align*}
\optPolicyPR \in \argmax_{\pi'\in \Pi} \min_{P \in P_\alpha,r \in \mathcal{R}_\alpha} \mathbb{E}_P^{\pi'}[\sum_t \gamma^t r(\state_t,\action_t)],
\end{align*}
where $\mathbb{E}^\pi_P$ is the expectation of policy $\pi$ when the dynamics are given by $P$. This relation explicitly shows that $\pi^*_{P,\alpha}$ is also optimal w.r.t. the worst model in the class $\mathcal{P}_\alpha, \mathcal{R}_\alpha$, which is convex and rectangular uncertainty set \cite{epstein2003recursive}, and formalizes the fact that PR-MDP is a specific instance of RMDP.

\subsection{Policy Iteration Schemes for PR-MDP}\label{sec: PI PR-MDP}
In this section, we analyze Policy Iteration (PI) schemes that solve \eqref{eq:PR-MDP}. Although a Value-Iteration procedure can be easily derived, we focus on the possible PI schemes. PI schemes are central to the currently used actor-critic approaches in continuous control, which we focus on in our experiments. We present two algorithms, Probabilistic Robust PI (Algorithm~\ref{alg:alpha robust PI}) and Soft Probabilistic Robust PI (Algorithm~\ref{alg: soft alpha robust PI}), and discuss the relation between the two.


\begin{figure*}
\begin{center}

\begin{minipage}[t]{0.4\textwidth}
\centering
 \begin{algorithm}[H]
	\caption{Probabilistic Robust PI }
	\label{alg:alpha robust PI}
	\begin{algorithmic}
		\STATE {\bf Initialize:} $\alpha,\bar{\pi}_0,k=0$
		\WHILE{not changing}
		\STATE $\pi_{k} \in \argmax_{\pi'} v^{\pi_{P,\alpha}^{\mathrm{mix}}(\pi',\bar{\pi}_{k})}$
		\STATE $\bar{\pi}_{k+1} \in \argmin_{\bar{\pi}}r^{\bar{\pi}} +\gamma P^{\bar{\pi}} v^{\pi_{P,\alpha}^{\mathrm{mix}}(\pi_{k},\bar{\pi}_k)}$
		\STATE $k \gets k+1$
		\ENDWHILE
		\STATE {\bf Return $\pi_{k-1} $}
		\STATE 
		\vspace{0.18cm}
	\end{algorithmic}
\end{algorithm}
\end{minipage}
\hspace{0.5cm}
\begin{minipage}[t]{0.4\textwidth}
\centering
 \begin{algorithm}[H]
	\caption{Soft Probabilistic Robust PI }
	\label{alg: soft alpha robust PI}
	\begin{algorithmic}
		\STATE {\bf Initialize:} $\alpha,\eta,\bar{\pi}_0,k=0$
		\WHILE{criterion is not satisfied}
		\STATE $\pi_{k} \in \argmax_{\pi'} v^{\pi_{P,\alpha}^{\mathrm{mix}}(\pi',\bar{\pi}_{k})}$
		\STATE $\bar{\pi}\in \argmin_{\bar{\pi}'}\Big \langle \bar{\pi}',\nabla_{\bar{\pi}} v^{\pi_{P,\alpha}^{\mathrm{mix}}(\pi_{k},\bar{\pi})}\mid_{\bar{\pi}=\bar{\pi}_k} \Big \rangle$
		\STATE $\bar{\pi}_{k+1} = (1-\eta)\bar{\pi}_{k} + \eta  \bar{\pi}$
		\STATE $k \gets k+1$
		\ENDWHILE
		\STATE {\bf Return $\pi_{k-1} $}
	\end{algorithmic}
\end{algorithm}
\end{minipage}
\end{center}
\end{figure*}

The Probabilistic Robust PI (PR-PI, Algorithm \ref{alg:alpha robust PI}) is a two player PI scheme adjusted to solving a PR-MDP (e.g., \citet{rao1973algorithms,hansen2013strategy}). PR-PI repeats two stages, (i)~given a fixed adversary strategy, it calculates the optimal counter strategy, and (ii)~it solves the 1-step greedy policy w.r.t. the value of the agent and adversary mixture policy. As suggested in \citet{shani2018revisiting}, Section 3.1, stage (i) may be performed by any MDP solver.

The Soft Probabilistic Robust PI (Soft PR-PI, Algorithm \ref{alg: soft alpha robust PI}) is updated using gradient information, unlike the PR-PI. Instead of updating the adversary policy using a 1-step greedy update, the adversary policy is updated using a Frank-Wolfe update \cite{frank1956algorithm}. The Franke-Wolfe update, similar to the gradient-projection method, finds a policy which is within the set of feasible policies; as, for instance, the gradient may produce policies out of the simplex. It works by finding the valid policy with the highest correlation, i.e., inner product, with the direction of gradient descent and performs a step towards it. As a convex mixture of two policies is a valid policy, the new policy is ensured to be a valid one.

Although the two algorithms might seem disparate, Soft PR-PI merely generalizes the `hard' updates of PR-PI to `soft' ones. This statement is formalized in the following proposition, which is a direct consequence of Theorem 1 in~\citet{scherrer2014local}, see proof in Appendix~\ref{supp: proof  1 step greedy gradients}.
\begin{proposition}\label{prop: 1 step greedy gradients} Let $\pi,\bar{\pi}$ be general policies. Then,
\begin{align*}
    &\argmin_{\bar{\pi}'\in \Pi}r^{\bar{\pi}'} +\gamma P^{\bar{\pi}'} v^{\pi_{P,\alpha}^{\mathrm{mix}}(\pi,\bar{\pi})} \\
&= \argmin_{\bar{\pi}'\in \Pi}\Big \langle \bar{\pi}',\nabla_{\bar{\pi}} v^{\pi_{P,\alpha}^{\mathrm{mix}}(\pi,\tilde{\pi})}\mid_{\tilde{\pi}=\bar{\pi}} \Big \rangle.
\end{align*}

\end{proposition}

Notice that the first single agent, 1-step improvement, has a solution in the set of deterministic policies (since the action space is a compact set and the argument is continuous in the action). Thus, $\bar{\pi}$ in Algorithm~\ref{alg: soft alpha robust PI} is exactly the 1-step greedy policy used in Algorithm~\ref{alg:alpha robust PI}. This suggests that for $\eta=1$  Algorithm~\ref{alg: soft alpha robust PI} is completely equivalent to Algorithm~\ref{alg:alpha robust PI}.

Generally, in two-player PI, the improvement stage amounts to solving a max-min, 1-step, decision problem. In PR-PI it is clearly not the case; in the improvement stage, a single agent, 1-step-greedy policy, is solved. Solving the latter is easier than solving the former, and it is a result of the specific structure of PR-MDP which does not generally hold, as will be demonstrated in Section \ref{sec: noisy AR}.

The following result shows that in both algorithms the value converges to the unique optimal value of the Nash-Equilibrium (see proof in Appendix~\ref{supp: PR-MDP contraction}).
\begin{theorem}\label{theorem: PR-MDP contraction}
    Denote by $v_k \eqdef v^{\pi_{P,\alpha}^{\mathrm{mix}}(\pi_{k},\bar{\pi}_k)}$. Then, for any $\eta \in (0,1]$, in Algorithm~\ref{alg: soft alpha robust PI}, $v_k$ contracts toward $\ValuePR^*$ with coefficient $(1-\eta+\gamma\eta)$, i.e., 
    \begin{equation*}
        {||v_k-\ValuePR^* ||_\infty \leq  (1-\eta+\gamma\eta)||v_{k-1}-\ValuePR^* ||_\infty} \enspace .
    \end{equation*}
\end{theorem}

Due to the equivalence of Algorithms \ref{alg:alpha robust PI} and \ref{alg: soft alpha robust PI} (when $\eta=1$), we get as a corollary that PR-PI converges toward the unique Nash-Equilibrium.

\begin{remark}
The solution method of the $\argmax$ and $\argmin$ in both Algorithms~\ref{alg:alpha robust PI} and \ref{alg: soft alpha robust PI} can be swapped and the convergence guarantees remain, e.g., $\bar \pi$ is the optimal solution to the MDP given $\pi$, whereas $\pi$ is updated using the 1-step greedy approach w.r.t. $\bar \pi$.
\end{remark}

\begin{remark}
Although Soft PR-PI converges slower than the non-soft version, it is reasonable to assume the former will be less sensitive to errors than the latter. Soft PR-PI can be seen as a generalization of Conservative PI (CPI) to solving PR-MDPs. CPI is known to be less sensitive to errors than other PI schemes \cite{scherrer2014local}. Nonetheless, the error analysis for Soft PR-PI is substantially different than the one CPI \cite{kakade2002approximately,scherrer2014approximate}. In Soft PR-PI, small changes in the adversarial policy may result in dramatic changes in the agent's policy. Thus, the $\gamma$-weighted state occupancy under a measure $\nu$, $d_\nu^{\pi_{P,\alpha}^{\mathrm{mix}}(\pi_k,\bar{\pi}_k)} = \sum_t \gamma^t \nu P^{\pi_{P,\alpha}^{\mathrm{mix}}(\pi_k,\bar{\pi}_k)}$, may change dramatically between iterations, whereas in CPI the change is smooth. We leave the error analysis for future work.
\end{remark}

\section{Noisy Action Robust MDP} \label{sec: noisy AR}
In this section we consider an alternative definition for action robustness. Instead of a stochastic perturbation in the policy space, as in Section~\ref{sec: probabilistic AR}, we consider a perturbation in the action space. To formally study such a perturbation we define the Noisy Action Robust MDP (NR-MDP), which, similarly to the PR-MDP, can be viewed as a zero-sum game (see Appendix~\ref{supp: NR-MDP and MG} 
for a formal mapping). We continue by establishing some properties of this MDP while highlighting important differences relative to the approach of PR-MDP.

\begin{restatable}{defn}{defNRMDP}
\label{def: NR-MDP}
Let $\alpha\in [0,1]$.  A Noisy Action Robust MDP is defined by the 5-tuple of an MDP (see Section \ref{sec:mdp}).  Let $\pi,\bar{\pi}$ be policies of an agent and an adversary. We define their noisy joint policy $\pi_{N,\alpha}^{\mathrm{mix}}(\pi,\bar{\pi})$ as 
\begin{align*}
  \forall s \in \mathcal{S}, \action \in \mathcal{A},\ &\pi_{N,\alpha}^{\mathrm{mix}}(\action \mid \state) \equiv  \E_{\substack{\baction\sim \pi(\cdot\mid s)\\
  \bar{\baction}\sim \bar{\pi}(\cdot\mid s)}} [ \mathds{1}_{\action = (1-\alpha)\baction+\alpha \bar{\baction}}],
\end{align*}
the relation is obtained by the fact that $\action\sim \pi,\bar{\action}\sim \bar{\pi}$.

Let $\pi$ be an agent policy. For NR-MDP, its value is defined by
${\ValueNR^\pi = \min_{\bar{\pi}\in \Pi} \mathbb{E}^{\pi_{N,\alpha}^{\mathrm{mix}}(\pi,\bar{\pi})}[\sum_t \gamma^t r(\state_t,\action_t)]},
$ where $\action_t\sim \pi_{N,\alpha}^{\mathrm{mix}}(\pi(\state_t),\bar{\pi}(\state_t))$. 
The optimal $\alpha$-noisy robust policy is the optimal policy of the NR-MDP
\begin{align}
\optPolicyNR \in \argmax_{\pi\in \mathcal{P}(\Pi)} \min_{\bar{\pi} \in \Pi} \mathbb{E}^{\pi_{N,\alpha}^{\mathrm{mix}}(\pi,\bar{\pi})}[\sum_t \gamma^t r(\state_t,\action_t)].\label{eq:NR-MDP}
\end{align}
The optimal noisy robust value is $\ValueNR^*=\ValueNR^{\optPolicyNR}$.
\end{restatable}


In simple terms; an optimal noisy robust policy is optimal w.r.t. a scenario, in which an adversary may change the agent's actions by adding bounded perturbations; the action performed on the system is $(1-\alpha)\action+\alpha\bar{\action}$, where $\bar{\action}$ is an action drawn from possibly adverserial distribution $\bar \pi$. The adversary's ability to add perturbations is controlled through the parameter $\alpha$. Each value of $\alpha$ defines a new continuous-action NR-MDP, where for $\alpha=0$ the adversary is unable to affect the system and the decision problem collapses to the standard, non-robust, MDP formulation.

The assumption on the structure of $\mathcal{A}$ is required, in order to ensure that the $\alpha$-mixture actions are valid actions, an assumption which holds naturally in the domain of continuous control. This approach formalizes a specific meaning for perturbation in the action space.

Although the approach of PR-MDP (Section \ref{sec: probabilistic AR}) and NR-MDP are closely related, they are not equivalent and important differences exist between the two. Unlike PR-MDP, for which a \emph{deterministic} stationary optimal policy exists, generally, for NR-MDP it is not the case. The optimal noisy robust policy, in the general case, is a \emph{stochastic policy} (see proof in Appendix~\ref{supp: noisy policy is stochastic}).

\begin{proposition}\label{prop: noisy policy is stochastic}
There exists an NR-MDP such that,
\begin{align*}
&\max_{\pi\in \Pi} \min_{\bar{\pi} \in \Pi} \mathbb{E}^{\pi_{N,\alpha}^{\mathrm{mix}}(\pi,\bar{\pi})}[\sum_t \gamma^t r(\state_t,\action_t)]\\
& < \max_{\pi\in \mathcal{P}(\Pi)} \min_{\bar{\pi} \in \Pi} \mathbb{E}^{\pi_{N,\alpha}^{\mathrm{mix}}(\pi,\bar{\pi})}[\sum_t \gamma^t r(\state_t,\action_t))].
\end{align*}
Furthermore, strong duality does not necessarily hold on the class of deterministic policies, $\Pi$.
\end{proposition}

The above proposition tells us that while it is often easier to focus on deterministic strategies (policies), when considering the NR-MDP scenario the optimal strategy may be stochastic. A similar notion has been shown to hold in non-cooperative matrix games \cite{nash1951non}, in which the optimal strategy is stochastic.

\subsection{Policy Iteration for NR-MDPs}\label{sec: noisy AR PI}

In section \ref{sec: PI PR-MDP}, we formulated PI schemes to solve PR-MDPs. Unlike two-player zero-sum PI \cite{rao1973algorithms,hansen2013strategy}, in PR-PI (Algorithm~\ref{alg:alpha robust PI}) a \emph{single} agent decision problem is solved, when the adversary policy $\bar \pi_{k+1}$ is updated. This structure is indeed unique to the PR-MDP, and does not hold when generalizing two-player zero-sum PI to solve NR-MDP.

Specifically, consider the two-player zero-sum PI that repeats the following two stages:
\begin{align*}
    &1.  \pi_k \in \argmax_{\pi\in \Pi} v^{\pi_{N,\alpha}^{\mathrm{mix}}(\pi,\bar{\pi}_k)}, \\
    &2.  \pi_k \in \argmin_{\bar{\pi} \in \mathcal{P}(\Pi)}\max_{\pi\in\Pi} r^{\pi_{N,\alpha}^{\mathrm{mix}}(\pi,\bar{\pi})}+P^{\pi_{N,\alpha}^{\mathrm{mix}}(\pi,\bar{\pi})}v^{\pi_{N,\alpha}^{\mathrm{mix}}(\pi_k,\bar{\pi}_k)}.
\end{align*}
$v^{\pi_{N,\alpha}^{\mathrm{mix}}(\pi,\bar{\pi}_k)}$ is the value of the joint policy $\pi_{N,\alpha}^{\mathrm{mix}}(\pi,\bar{\pi}_k)$, ${r^{\pi_{N,\alpha}^{\mathrm{mix}}(\pi,\bar{\pi})}(s) = \E_{\action\sim\pi,\bar \action\sim \bar \pi}[r(s,(1-\alpha)\action+\alpha \bar{\action})]}$, and $ P^{\pi_{N,\alpha}^{\mathrm{mix}}(\pi,\bar{\pi})}(s,s') = \E_{\action\sim\pi,\bar \action\sim \bar \pi}[P(s\mid s,(1-\alpha)\action+\alpha \bar{\action})]$ are the induced reward and dynamics from by $\pi_{N,\alpha}^{\mathrm{mix}}(\pi,\bar \pi)$. Following similar lines of proof as in \citet{hansen2013strategy} or as in Theorem \ref{theorem: PR-MDP contraction}, a similar $\gamma$-contraction result may be achieved for the NR-MDP, e.g., ${||v_k- \ValueNR^*||_\infty \leq \gamma||v_{k-1}-\ValueNR^* ||_\infty}$.

In such an algorithm, stage (1) is performed by solving an MDP, as in PR-PI. However, stage (2) requires solving a 1-step min-max problem. For general reward and transition probabilities it cannot be solved by solving a single-agent decision problem, as in the second stage of PR-PI (Algorithm~\ref{alg:alpha robust PI}). Furthermore, the solution of stage (2) cannot be achieved by a single-call to a gradient oracle as in Proposition \ref{prop: 1 step greedy gradients} (we elaborate the discussion in Appendix~\ref{supp: PI for NR-MDP}). 


Regardless of these differences, in Section \ref{sec: experiments}, we will use the approach of Soft PR-PI and offer DRL algorithms to solve both the PR and NR MDPs. While the approach we consider in Section \ref{sec: experiments} should be understood as a heuristic for solving NR-MDP, it is based on Algorithm~\ref{alg: soft alpha robust PI}, which guarantees convergence for PR-MDP in the error-free case.



\section{Related Work} \label{sec: related work}

\textbf{Robust RL:} Traditional works in RL, such as \citet{nilim2005robust} and \citet{iyengar2005robust} have provided efficient algorithms for solving Robust MDPs, with uncertainty in the transition probabilities. \citet{mannor2012lightning} extended their approach to non-rectangular uncertainty sets, e.g., coupled uncertainty sets. However, these approaches are limited to solutions in the tabular case. Additionally, a connection between robustness and generalization has been suggested \cite{xu2009robustness,xu2012robustness}, while it is not clear how this holds in RL, we believe that there lies a similar yet complex connection between the two concepts.

\textbf{Control:} Obtaining robust policies in continuous control problems has been extensively investigated in the past. Most closely related to our work, are  max-min Robust Control approaches (e.g., \citet{bemporad2003min,kerrigan2004feedback,de2006feedback}. In this line of work, a control policy which is robust w.r.t. deterministic perturbations is calculated. There, the max-min problem is solved via Linear program, Quadratic program or by an explicit tree-search. Here, we focus on PI, and gradient based, schemes to solve a more specific problem; action robust policies. Furthermore, and to the best of our knowledge, in this line of works,  discussion on the existence of strong-duality does not exists (i.e., as Proposition \ref{prop:PR-MDP and determinstic optimal policy} and \ref{prop: noisy policy is stochastic} assert for PR- and NR-MDPs).



\textbf{Robust Supervised Learning:} Similar to the Robust MDPs framework, robustness to adversarial examples/attacks \cite{szegedy2013intriguing} is a measure of robustness in supervised learning. A method of learning robust classifiers is through Generative Adversarial Networks \cite{goodfellow2014generative}. Similar to our approach, when using GANs for robustness, an adversary learns to create small perturbations in the input data in an attempt to cause a mis-classification \cite{xiao2018generating,samangouei2018defense,kurakin2018adversarial}. While these methods work well in practice, they generally lack convergence proofs and should thus be treated as heuristics.




\section{Experiments} \label{sec: experiments}
\subsection{Method}
Our approach adapts the Soft PR-PI algorithm to the high dimensional scenario. While in the tabular case we may use an MDP solver, which produces the optimal policy; when considering parametrized policies, e.g., neural networks, a dual-gradient approach is taken. In this approach, both the Actor and the Adversary are trained using gradient descent; as it is hard to measure convergence - we train the actor for $N$ gradient steps followed by a single adversary step.

We focus on a robust variant of DDPG which we call Action-Robust DDPG (AR-DDPG, see Appendix~\ref{apndx: gradients proof}, 
Algorithm~\ref{alg:robust_ddpg}). 
DDPG \cite{lillicrap2015continuous} trains an actor to predict an action for each state $\mu_\theta:\mathcal{S}\rightarrow \mathcal{A}$ (i.e., a deterministic policy). In AR-DDPG we train two networks, deterministic policies, the actor and adversary, denoted by $\mu_\theta$ and $\bar \mu_{\bar \theta}$. Similarly to DDPG, a critic is trained to estimate the $q$-function of the joint-policy. For PR-MDP (Definition \ref{def: PR-MDP}), the joint policy is 
\begin{equation}\label{eq: PR joint policy determinstic}
\pi^{\mathrm{mix}}_{P,\alpha}(u\!\mid\! s; \theta,\bar \theta) \!=\! (1-\alpha)\delta(u-\mu_\theta(s))+\alpha\delta(u - \bar \mu_{\bar \theta}(s)),
\end{equation}
whereas for NR-MDP (Definition \ref{def: NR-MDP}), the joint policy is,
\begin{equation}\label{eq: NR joint policy determinstic}
\pi^{\mathrm{mix}}_{N,\alpha}(u\!\mid\! s; \theta,\bar \theta)\! = \!\delta(u-((1-\alpha)\mu_\theta(s) + \alpha \bar \mu_{\bar\theta}(s))) ,
\end{equation}
where $\delta(\cdot)$ is the Dirac delta function.

The following result generalizes DPG \cite{silver2014deterministic} for both PR and NR-MDPs. i.e., it establishes how to update $\theta$ and $ \bar\theta$ using a deterministic gradient based method.

\begin{proposition}\label{prop: experiments proof}
    Let $\mu_\theta, \bar\mu_{\bar \theta}$ be the agent's and adversary's deterministic policies, respectively. Let $\pi(\mu_\theta, \bar\mu_{\bar \theta})$ be the joint policy given the agent and adversary policies. i.e., for PR-MDP $\pi = \pi^{\mathrm{mix}}_{P,\alpha}$ \eqref{eq: PR joint policy determinstic}, and for NR-MDP $\pi = \pi^{\mathrm{mix}}_{N,\alpha}$ \eqref{eq: NR joint policy determinstic}.
    
    Let ${J(\pi(\mu_\theta, \bar\mu_{\bar \theta})) = \mathbb{E}_{\state \sim \rho^{\pi}} [v^{\pi}(s)]}$ be the performance objective. The gradient of the actor and adversary parameters, for both PR- and NR-MDP is:
    \begin{align*}
        \nabla_\theta J(\pi(\mu_\theta, \bar\mu_{\bar \theta})) &= (1\! -\! \alpha) \mathbb{E}_{\state \sim \rho^{\pi}} [\nabla_\theta \mu_\theta (\state) \nabla_{\action} Q^{\pi} (\state, \action)] \enspace ,\\
        \nabla_{\bar \theta} J (\pi(\mu_\theta, \bar\mu_{\bar \theta})) &= \alpha \mathbb{E}_{\state \sim \rho^{\pi}} [\nabla_{\bar \theta} \bar \mu_{\bar{\theta}} (\state) \nabla_{\bar{\action}} Q^{\pi} (\state, \bar \action)] \enspace.
    \end{align*}
    
    where for the PR-MDP we have $\action = \mu_\theta (\state)$ and $\bar \action = \bar \mu_{\bar \theta} (\state)$, and for the NR-MDP $\action = \bar\action = (1 - \alpha) \mu_\theta (\state) + \alpha \bar \mu_{\bar \theta} (\state)$.
    
\end{proposition}
A proof, example algorithm and block diagram are provided in Appendix~\ref{apndx: gradients proof}.

In order to validate our approach, we consider several MuJoCo domains \cite{todorov2012mujoco}. MuJoCo contains several continuous control problems, such as robotic manipulation, in which we may test the ability of our approach to produce robust policies. Intuitively, our Probabilistic operator is correlative to the occurrence of large \textit{abrupt} forces, e.g., someone suddenly pushes the robot, whereas the Noisy operator is correlative to \textit{mass} uncertainty, e.g., the robot is heavier or lighter.

Our evaluation is split into two parts, we begin by comparing the various hyper-parameters and how they affect the performance of both the NR and PR-MDP approaches. This evaluation is performed extensively on a single domain, the Hopper-v2 task, and the figures are provided in the appendix. We then compare the best performing variants across unseen domains. By doing so we test the transferability of these hyper-parameters across domains.

\subsection{Theory versus Practice}

Our theoretical approach, Soft PR-PI (Algorithm~\ref{alg: soft alpha robust PI}), is proven for the PR-MDP. The algorithm is based on a dynamic programming approach, (i) given a fixed adversary policy, solves the optimal agent's policy, (ii) updates the adversary policy using gradients.
\vspace{-0.25cm}
\begin{enumerate}
    \item While in theory, for the PR criterion, there exists a deterministic optimal policy - this does not necessarily hold for the NR case (Proposition~\ref{prop: noisy policy is stochastic}). Thus searching over the space of deterministic policies is sub-optimal.
    \item \vspace{-0.2cm}Theoretical approaches in general require exact computation, however, in practice, we use function approximation schemes, e.g., deep neural networks. As such, convergence can not be ensured and the approach should be seen as a heuristic.
\end{enumerate}
Regardless of these differences, we based the empirical approach for both PR and NR-MDPs on Algorithm~\ref{alg: soft alpha robust PI}.

\begin{figure}[t]
\centering
\hfill\mbox{}$\enspace\enspace\enspace$NR-MDP \hfill $\enspace\enspace\enspace\enspace\enspace\enspace\enspace\enspace$ PR-MDP\hfill\mbox{}\\
\begin{subfigure}
    \centering
	\includegraphics[width=39mm]{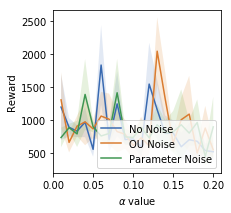}
\end{subfigure}%
\begin{subfigure}
    \centering
	\includegraphics[width=39mm]{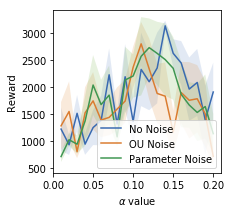}
\end{subfigure}%
\caption{Hopper-v2: Performance of both the NR and PR-MDP criteria as a function of the uncertainty $\alpha$.}
\label{fig:hopper_noisy_ablation}
\end{figure}

\begin{figure*}[t]
\begin{tabular}{>{\centering\arraybackslash}m{.15\linewidth} >{\centering\arraybackslash}m{.25\linewidth} >{\centering\arraybackslash}m{.25\linewidth} >{\centering\arraybackslash}m{.25\linewidth}}
  & Baseline & NR-MDP & PR-MDP \\ 
 Hopper & \includegraphics[width=40mm]{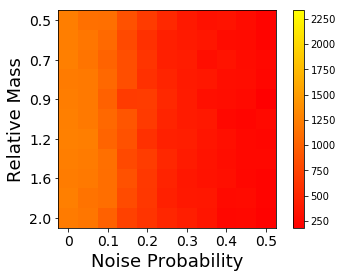} & \includegraphics[width=40mm]{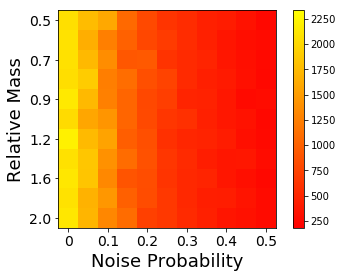} & \includegraphics[width=40mm]{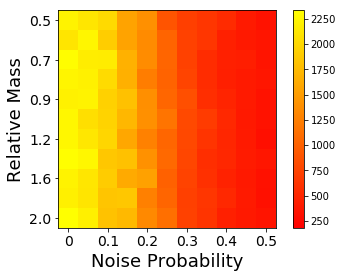} \\
 Walker2d & \includegraphics[width=40mm]{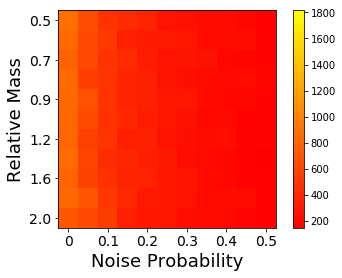} & \includegraphics[width=40mm]{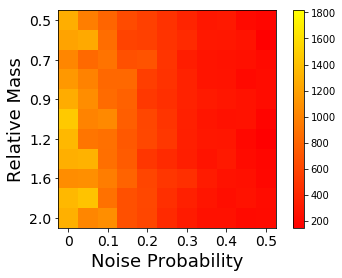} & \includegraphics[width=40mm]{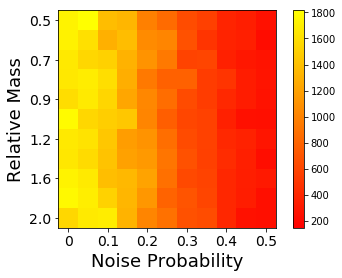} \\
 Humanoid & \includegraphics[width=40mm]{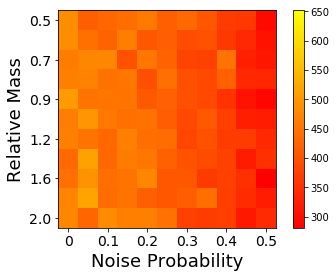} & \includegraphics[width=40mm]{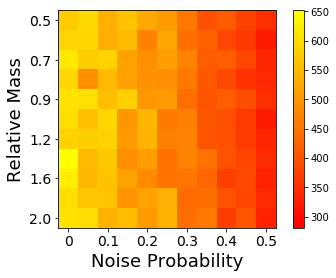} & \includegraphics[width=40mm]{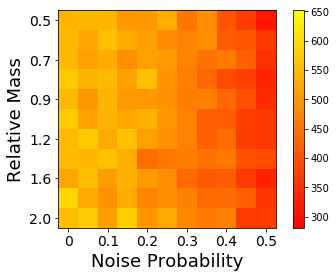} \\
 InvertedPendulum & \includegraphics[width=40mm]{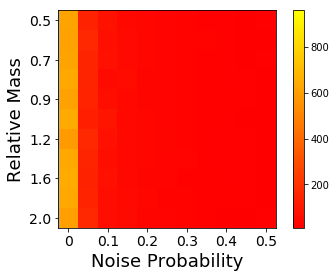} & \includegraphics[width=40mm]{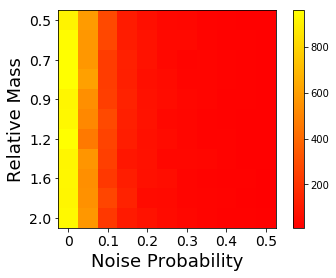} & \includegraphics[width=40mm]{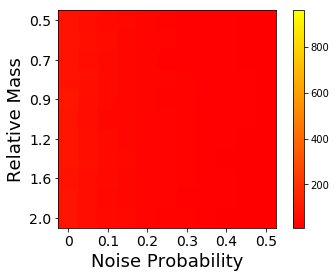}
\end{tabular}
\caption{Robustness to model uncertainty. Noise probability denotes the probability of a randomly sampled noise being played instead of the selected action.}
\label{fig:all_mujoco_model}
\end{figure*}

\subsection{Hyperparameter Ablation}
\begin{table}[h]
\begin{center}
\caption{Hyper-parameters considered.}\label{table:params}
\begin{tabular}{|l|l|}
	\hline
    \\[-1em]
    $\alpha$ values & 0.01, 0.05, 0.1, 0.15 and 0.2\\
    \hline
    \\[-1em]
    Actor update steps $N$ & 2, 5, 10 and 20\\
    \hline
\end{tabular}
\end{center}
\end{table}
The hyper-parameters we consider are shown in Table~\ref{table:params}. In addition, we consider 3 exploration schemes: noiseless (on-policy exploration), Ornstein Ulenbeck (OU, \citet{uhlenbeck1930theory}) and Parameter space noise \cite{plappert2017parameter}. Each configuration, is trained on 5 random seeds and the final policy, once the training is concluded, is evaluated across 100 episodes. The evaluation is performed without adversarial perturbations, on a range of mass values not encountered during training, i.e., we test the ability of the action robust approach to produce policies which are \emph{robust to model uncertainty}. The baseline we compare to, is DDPG with parameter space noise for exploration, which performed best in our experiments.

The extensive comparison is presented in the appendix, however the main conclusion is shown in Figure~\ref{fig:hopper_noisy_ablation}. While there is a clear correlation between the value of $\alpha$ and the performance of the PR-MDP criteria, e.g., an optimal value is attained at $\alpha \in [0.1, 0.15]$ and deviating from this range results in performance deterioration - this is not the case for the NR-MDP. Although the NR-MDP often attains competitive results, it is not clear how the various parameters affect it. We conclude that the for our simple gradient based approach, the PR approach exhibits a more stable behavior than the NR approach.

Specifically, for the PR-MDP we decided to use Parameter space noise with $\alpha=0.1$ and a ratio of 10:1. Even though there are certain configurations under which the OU noise variant outperformed the Parameter space noise, we decided on the latter as it exhibited higher stability and is thus more likely to transfer easily to new domains. Similarly, a large $\alpha$ provides greater control to the adversary, as such we decided on a more conservative value of $0.1$.

For the NR-MDP this selection process is somewhat harder; as slight changes in the hyper-parameters may result in radical changes in the performance. We selected the OU noise combined with $\alpha=0.1$ and a training ratio of 1:1.

An interesting insight is that in the PR-MDP criteria, the adversary induces enough noise for exploration (Figure~\ref{fig:hopper_noisy_ablation}, PR-MDP - No Noise plot). This can be seen when observing the `no noise' experiments, which show that the PR-MDP approach outperforms the baseline even without additional exploration noise.



\subsection{Testing on various MuJoCo domains}

Figure~\ref{fig:all_mujoco_model} presents our results, on various MuJoCo domains (additional results in Appendix~\ref{apndx: empirical results}). 
It is apparent that while in the Hopper-v2 domain, the PR-MDP outperformed the NR-MDP criterion; this does not hold on all domains. Moreover, in most of the domains, both operators outperform the baseline, both in terms of robustness and in terms of performance in the absence of perturbations. While the optimal parameters may differ across domains; our results show that, in most cases, the parameters transfer across domains and result in improved performance without additional tuning.

\textbf{Failures:} It is also important to acknowledge the scenarios in which our algorithm does not outperform the baseline. Such an example is the InvertedPendulum domain, in which the performance of the PR-MDP was found to be inferior to that of its non-robust counterpart. We find two possible explanations for this phenomenon (i) the parameter tuning is performed on the Hopper domain (as opposed to selecting the optimal hyper-parameters per each domain). As each domain is different, it is plausible that good hyper-parameters in a certain domain would not be good in all domains. (ii) Specifically in the InvertedPendulum domain, where the task is to balance a pole, an adversary which is too strong (large $\alpha$ value) prevents the agent from successfully solving the task.

\begin{figure}[h]
\centering
\hfill\mbox{}$\enspace\enspace\enspace$NR-MDP \hfill $\enspace\enspace\enspace\enspace\enspace\enspace\enspace\enspace$ PR-MDP\hfill\mbox{}\\
\begin{subfigure}
    \centering
	\includegraphics[width=40mm]{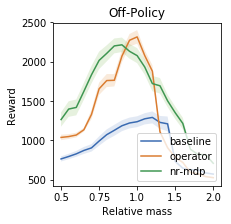}
\end{subfigure}%
\begin{subfigure}
    \centering
	\includegraphics[width=40mm]{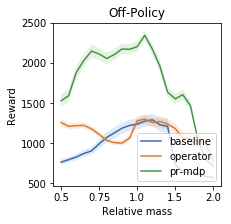}
\end{subfigure}%
\\
\begin{subfigure}
    \centering
	\includegraphics[width=40mm]{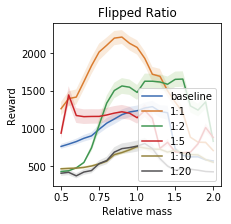}
\end{subfigure}%
\begin{subfigure}
    \centering
	\includegraphics[width=40mm]{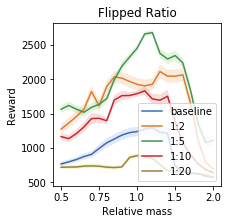}
\end{subfigure}%
\caption{Diving Deeper: (Up) Testing Off-Policy Action-Robustness, and (Down) Solving the MaxMin operator.}
\label{fig:hopper_noisy_additional}
\end{figure}

\subsection{Diving Deeper}\label{sec: diving deeper}
We attempt to analyze the behavior of our criteria (Figure~\ref{fig:hopper_noisy_additional}) by asking two questions: (i) Does the performance increase due to the added perturbations from the adversary, or does the operator itself induce a prior, e.g., regularization, on the policy which leads to improved performance. (ii) How close is the empirical behavior to its theoretical counterpart.

\textbf{Off-Policy Action Robustness:} In previous experiments, during training, the action was drawn from the joint policy of the agent and adversary, where the joint policy is specified in the PR and NR-MDP approaches (see Definition \ref{def: PR-MDP},\ref{def: NR-MDP}).

A natural alternative approach is to \emph{act with the actor's policy}, yet, to acquire an action-robust policy in an off-policy fashion. Meaning, use the same algorithms while obtaining the data without the effect of the adversary. A possible advantage of such an approach is minimizing the number of bad actions (since the adversary does not intervene), while still benefiting from the presence of robust learning.

Figure \ref{fig:hopper_noisy_additional} presents the results of this experiment. For the NR-MDP, it seems that the operator itself, i.e., the training is what results in the performance improvement; whereas the adversarial exploration amount to a small increase in stability. Surprisingly, an opposite effect is observed for the PR-MDP. There, the combination of adversarial exploration and the operator are both required in order to attain the performance increase.



\textbf{Does MaxMin equal MinMax?} While so far we trained our agent through $N$ actor updates followed by a single adversary gradient update, this corresponds to the MinMax operator, in theory the opposite should result in an identical performance (Proposition~\ref{prop:PR-MDP and determinstic optimal policy}) for the PR-MDP approach, and to deteriorate  the performance for the NR-MDP approach (Proposition~\ref{prop: noisy policy is stochastic}).

Experimentally (Figure~\ref{fig:hopper_noisy_additional}) the results show that as opposed to the theoretical analysis, a `stronger' adversary does result in performance degradation. This could be due to two possible factors: (i) as we trained for the same number of steps for both scenarios, it means that in this case the actor receives less gradient update steps, and/or (ii) it could be that increasing the convergence of the adversary results in faster convergence to a sub-optimal solution (w.r.t. the actor).




\section{Summary}
We have presented two new criteria for robustness, the Probabilistic and Noisy action Robust MDP, related each to real world scenarios of uncertainty and discussed the theoretical differences between both approaches. Additionally; we developed the Soft PR-PI (Algorithm~\ref{alg: soft alpha robust PI}), a policy iteration scheme for solving PR-MDPs. Building upon the Soft PR-PI algorithm, we presented a deep reinforcement learning approach, which is capable of solving our criteria. We compared both criteria, analyzed how the various hyper-parameters affect the behavior and how the empirical results correlate (and occasionally contradict) with the theoretical approach. Most importantly, we notice that not only does training with our criteria result in robust policies, but our approach improves performance even in the absence of perturbations.

Lastly, for solving an action-robust policy, there is  \emph{no need in providing an uncertainty set}. The approach requires only a scalar value, namely $\alpha$ (or possibly a state-dependent $\alpha(s)$), which \emph{implicitly} defines an uncertainty set (see Section~\ref{sec: PR MDP and RMDPs}). This is a major advantage compared to standard robust approaches in RL and control, which, to the best of our knowledge, require a distribution over models or perturbations. Of course, this benefit is also a restriction - the Action Robust approach is unable to handle any kind of worst-case perturbations. Yet, due to its simplicity, and its demonstrated performance, it is worthwhile to be considered by an algorithm designer.

\section{Acknowledgements}
The authors would like to thank Bruno Scherrer, Esther Derman and Nadav Merlis for the fruitful discussions and help during the work on this paper.


\bibliography{action_robust_bib}
\bibliographystyle{icml2019}


\onecolumn
\appendix


\section{Discounted Markov Games}\label{supp: zero sum MG}

\subsection{Preliminaries}\label{supp sec: pre MG}
We define the framework of discounted, two-player zero-sum Markov Games (MG) with finite state space and continuous action space. A MG is determined by the 5-tuple $(\mathcal{S},\mathcal{A},\mathcal{B},P,R,\gamma)$ \cite{patek1997stochastic}. Here $\mathcal{S}$ is a finite state space, $\mathcal{A}$ and $\mathcal{B}$ are compact subsets of $ \mathbb{R}^A$, which represent the agent and adversary, respectively. For any $(s,a,b)\in \mathcal{S}\times \mathcal{A}\times \mathcal{B}$ let the dynamics $P=P(\cdot \mid s,a,b)$ be a probability measure on $\mathcal{S}$, and let the reward function $r(s,a,b)$ be a bounded measureable function on $\mathcal{A}\times \mathcal{B}$ for any $s\in \mathcal{S}$. Consider a strategy of the players $\mu,\nu$, where both are probability measures over Borel sets of $\mathcal{A},\mathcal{B}$, respectively. Let $r^{\mu,\nu}\in \mathbb{R}^{|\mathcal{S}|}$ where $r^{\mu,\nu}(s)\eqdef \E_{a\sim \mu,b\sim \nu}[ r(s,\mu,\nu)]$, and the dynamics $P^{\mu,\nu} \in \mathbb{R}^{|\mathcal{S}|\times |\mathcal{S}|}$, where $P^{\mu,\nu}_{i,j}\eqdef \E_{a\sim \mu,b\sim \nu}[ P(s_j\mid s_i,\mu,\pi_B)]$ and is a stochastic matrix. Following notation from \citet{maitra1970stochastic}, we denote $P_A$ and $P_B$  as the set of probability measures on the Borel sets of $\mathcal{A}$ and $\mathcal{B}$, respectively.

\begin{defn}\label{def: optimal values on MG}
The value of fixed strategy $\mu,\nu$ is given by $v^{\mu,\nu} = \sum_{t=0}^\infty \gamma^t (P^{\mu,\nu})^tr^{\mu,\nu}$.  Given a fixed $\nu\in P_B$ the value of the optimal counter strategy of player $A$ is $v^\nu=\sup_{\mu \in P_A} v^{\mu,\nu}$. Accordingly, for a fixed  $\mu\in P_A$ the value of the optimal counter strategy of player $B$ is $v^\mu = \inf_{\nu \in P_B} v^{\mu,\nu}$. Furthermore, if the $\sup$ and $\inf$ are attainable, we refer to $\argmin_{\nu \in P_B} v^{\mu,\nu}$ and $\argmax_{\mu \in P_A} v^{\mu,\nu}$ as optimal counter strategies to $\mu$ and $\nu$, respectively.
\end{defn}

We make the following assumptions on the dynamics and reward functions.
\begin{assumption}\label{assumptions: MG}
\hfill
\begin{itemize}
    \item Both $\mathcal{A},\mathcal{B}$ are compact metric spaces.
    \item For any $s\in \mathcal{S}$ the reward $r$ is continuous and bounded function on $\mathcal{A}\times\mathcal{B}$.
    \item  For any $s\in \mathcal{S}$,   whenever $(a_n,b_n)\rightarrow(a,b)$, where $(a_n,b_n),(a,b)\in \mathcal{A}\times\mathcal{B}$, then $P(\cdot\mid s,a_n,b_n)$ converges weakly to $P(\cdot\mid s,a,b)$.
\end{itemize}
\end{assumption}

In the rest of the section we follow \cite{patek1997stochastic}[Section 2-3] that analyzed zero-sum MG for stochastic shortest paths, while performing minor modifications for the discounted and continuous action-space setup.

Define the following Bellman operators.
\begin{defn}\label{def: supp Bellman operators MG}
Let $P_A$ and $P_B$ be the set of all probability measures on the Borel Sets of $\mathcal{A}$ and  $\mathcal{B}$, respectively, $\mu\in P_A,\nu\in P_B$, and let $v\in \mathbb{R}^{|\mathcal{S}|}$.  The Bellman operator, and Fixed-Policy Bellman operators are according to the following.
\begin{align*}
&T^{\mu,\nu}v = r^{\mu,\nu} + \gamma P^{\mu,\nu}v,\\
T^{\mu}v = \min_{\nu\in P_B} &\left( r^{\mu,\nu}+ \gamma P^{\mu,\nu}v \right),\ \bar{T}^{\nu}v = \max_{\mu\in P_A}\left( r^{\mu,\nu} + \gamma P^{\mu,\nu}v \right)\\
T v = \max_{\mu\in P_A} \min_{\nu\in P_B} & \left( r^{\mu,\nu} + \gamma P^{\mu,\nu}v\right),\ \bar{T} v  = \min_{\nu\in P_B} \max_{\mu\in P_A} \left( r^{\mu,\nu} +  \gamma P^{\mu,\nu}v \right),
\end{align*}
where equality holds component-wise.
\end{defn}

Notice that the $\max$ and $\min$ are attainable since $ P_A,P_B$ are compact sets. Furthermore, by \citet{maitra1970stochastic}[Lemma 2.2] and under Assumption \ref{assumptions: MG}, both the $\max$ and $\min$ are continuous and bounded. Thus, we can replace $\sup\inf$ and $\inf\sup$ by corresponding $\max$ and $\min$. 

We have the following important lemma.
\begin{lemma} \label{lemma: min max is max min}
For any bounded $v\in \mathbb{R}^{|\mathcal{S}|}$, $T v = \bar{T}v$.
\end{lemma}
\begin{proof}
Following similar arguments as in \citet{maitra1970stochastic}, Equation 2, and using Sion's minimax theorem \cite{sion1958general}[Theorem 3.4], for any $s\in\mathcal{S}$ we have that,
\begin{align*}
     \sup_{\mu\in P_A} \inf_{\nu\in P_B}  r^{\mu,\nu}(s)+ P^{\mu,\nu}v(s)  =  \inf_{\nu\in P_B} \sup_{\mu\in P_A}  r^{\mu,\nu}(s) + P^{\mu,\nu}v(s).
\end{align*}
Since  $P_A,P_B$ are compact and $r^{\mu,\nu} + P^{\mu,\nu}v$ is bounded and continuous on $\mathcal{A}\times\mathcal{B}$ for any $s\in\mathcal{S}$, the $\sup,\inf$ can be replaced by $\min,\max$ (e.g., by \citet{maitra1970stochastic}[Lemma 2.2]). 
\end{proof}

The analysis in \citet{patek1997stochastic} is based on assumption R, which results in $T v =\bar{T}v$. Since we allow the agents to use mixed-strategies, according to Lemma \ref{lemma: min max is max min}, we obtain $T v =\bar{T}v$ in our setup as well. Furthermore, since we use discounted MG, assumption SSP in \citet{patek1997stochastic} is also satisfied. Every strategy $(\mu,\nu)$ is proper; it terminates with probability one, as the discount factor ($\gamma$) is smaller than 1.

\begin{lemma}\label{lemma: supp everything is gamma contraction}
$T^{\mu,\nu},\ T^{\mu},\ \bar{T}^{\nu}, T$ are $\gamma$ contractions in the sup-norm.
\end{lemma}
\begin{proof}
We follow similar technique as in \citet{patek1997stochastic}, adjusted to our setup. Let $v_1,v_2\in \mathbb{R}^{|\mathcal{S}|}$. Then,
\begin{align*}
    T^{\mu,\nu} v_1-T^{\mu,\nu}v_2 = \gamma P^{\mu,\nu}(v_1-v_2)\leq \gamma P^{\mu,\nu}{\bf 1}||v_1-v_2||_\infty = \gamma {\bf 1}||v_1-v_2||_\infty,
\end{align*}
where ${\bf 1}$ is the one vector. The last relation holds since $P^{\mu,\nu}$ is a stochastic matrix and thus $P^{\mu,\nu}{\bf 1}={\bf 1}$. By repeating the same argument for $T^{\mu,\nu} v_2-T^{\mu,\nu}v_1$ and taking the sup-norm we conclude that $|| T^{\mu,\nu} v_1-T^{\mu,\nu}v_2 ||_\infty\leq \gamma ||  v_1-v_2 ||_\infty$.

We now prove similar result on $T^{\mu}$. Let $\nu,\nu'\in P_B$ such that $T^{\mu}v_1 = T^{\mu,\nu}v_1,\ T^{\mu}v_2 = T^{\mu,\nu'}v_2$. Then,
\begin{align*}
    &T^{\mu}v_1-T^{\mu}v_2 \leq T^{\mu,\nu}v_1-T^{\mu,\nu}v_2,\\
    &T^{\mu}v_2-T^{\mu}v_1 \leq T^{\mu,\nu'}v_1-T^{\mu,\nu'}v_2.
\end{align*}
By taking the sup-norm and using the fact $T^{\mu,\nu}$ is a $\gamma$-contraction, we conclude that  $T^{\mu}$ is also a $\gamma$-contraction. Similar argument establishes that $\bar{T}^{\nu}$ is a $\gamma$-contraction. 

Lastly, let $\mu\in P_A$ such that $T v_2 = T^\mu v_2$, and $\nu\in P_B$ such that $T^\mu v_1 = T^{\mu,\nu}v_1$. Then,
\begin{align*}
    T v_1-T v_2 &= T v_1-T^\mu v_2 \\
    &\leq  T^\mu v_1-T^\mu v_2 \\
    &=  T^{\mu,\nu} v_1-T^\mu v_2\\
    &\leq  T^{\mu,\nu} v_1-T^{\mu,\nu} v_2.
\end{align*}
Similar argument leads to $T v_2-T v_1 \leq T^{\mu,\nu} v_2-T^{\mu,\nu} v_1$ for properly defined $\mu,\nu$. Again, by taking the sup norm and using the fact that $T^{\mu,\nu}$ is a $\gamma$-contraction we conclude the proof.
\end{proof}

The following propositions relate the fixed-point of $T_\mu,\bar{T}^\nu$ to the values and policies defined in \ref{def: optimal values on MG}. Furthermore, the last one establishes the fact the zero-sum MG has value.

\begin{proposition}\label{prop supp: fixed policy}The following claims hold.

\begin{itemize}
    \item Let $\mu\in P_A,\nu\in P_B$ be stationary policies. The value $v^{\mu,\nu}$ is the fixed point of the operator $T^{\mu,\nu}$, $v^{\mu,\nu} = T^{\mu,\nu}v^{\mu,\nu}$.
    \item Given a policy $\nu\in P_B$, $v^{\nu} = \sup_{\mu\in P_A}$ is the unique fixed point of $\bar{T}^\nu$. Furthermore, the $\sup$ is attainable in the set $A$.
    \item Given a policy $\mu\in P_A$, $v^{\mu} = \inf_{\nu\in P_B}$ is the unique fixed point of $T_\mu$. Furthermore, the $\inf$ is attainable in the set $B$.    
\end{itemize}
\end{proposition}
\begin{proof}
    The proof of the first claim is standard, e.g., \citet{puterman1994markov}[Section 6.1].
    By fixing a policy for any of the players the problem amounts for solving a single agent MDP (e.g., \citet{puterman1994markov}). Due to Assumption \ref{assumptions: MG}, the reward and dynamics of the MDP are also continuous and bounded. Since the action set in compact for both player $A$ and $B$, we can use \citet{puterman1994markov}[Theorem 6.2.10] and conclude the proof. 
\end{proof}

\begin{proposition}\label{prop: supp minmax and equilibrium value}
The unique fixed point $v^*=T v^*$ is also the equilibrium value of the zero-sum MG, ${v^* = \sup_{\mu\in P_A} \inf_{\nu\in P_B} v^{\mu,\nu} =  \inf_{\nu\in P_B}\sup_{\mu\in P_A} v^{\mu,\nu} }$, thus, the MG has a well defined value.

Furthermore,  the stationary policies $\mu\in P_A,\nu\in P_B$ for which $v^* =\bar{T}v^* =T v^* = T^{\mu,\nu} v^*$ are in Nash-Equilibrium, and satisfy $v^{\mu',\nu^*} \leq  v^* \leq  v^{\mu^*,\nu}$ for any $\nu'\in P_B,\ \mu'\in P_A$.
\end{proposition}

\begin{proof}
    See proof \citet{patek1997stochastic}[Proposition 3.2].
\end{proof}

\subsection{Policy Iteration and Soft Policy Iteration for Zero-Sum Markov Games}

\begin{figure*}
\hspace{1cm}
\begin{minipage}[t]{0.4\textwidth}
\centering
 \begin{algorithm}[H]
	\caption{Zero-Sum Markov-Game PI }
	\label{alg supp: general MG PI}
	\begin{algorithmic}
		\STATE {\bf Initialize:} $\nu_0,k=0$
		\WHILE{stopping criterion is not satisfied}
		\STATE $\mu_{k} \in \argmax_{\mu} v^{\mu,\nu_k}$
		\STATE $\nu_{k+1} \in \argmin_{\nu}\bar{T}^{\nu} v^{\mu_k,\nu_k}$
		\STATE $~~ k ~~ ~\gets k+1$
		\ENDWHILE
		\STATE {\bf Return $\pi_{k-1} $}
		\STATE \vspace{0.1cm}
	\end{algorithmic}
\end{algorithm}
\end{minipage}
\hspace{1cm}
\begin{minipage}[t]{0.4\textwidth}
\centering
 \begin{algorithm}[H]
	\caption{Soft Zero-Sum Markov-Game PI}
	\label{alg supp: general MG soft PI}
	\begin{algorithmic}
		\STATE {\bf Initialize:} $\nu_0,k=0,\eta\in(0,1]$
		\WHILE{stopping criterion is not satisfied}
		\STATE $\mu_{k} \in \argmax_{\mu} v^{\mu,\nu_k}$
		\STATE $\nu' \in \argmin_{\nu}\bar{T}^{\nu} v^{\mu_k,\nu_k}$
		\STATE $\nu_{k+1} = (1-\eta)\nu_k + \eta \nu'$
		\STATE $~~ k ~~ ~\gets k+1$
		\ENDWHILE
		\STATE {\bf Return $\pi_{k-1} $}
	\end{algorithmic}
\end{algorithm}
\end{minipage}
\end{figure*}

In this section, we formulate two PI schemes that solve a zero-sum MG. The Zero-Sum MG PI scheme (see Alg. \ref{alg supp: general MG PI}) is a well known one \cite{hoffman1966nonterminating,rao1973algorithms,hansen2013strategy}. 

The Soft Zero-Sum MG PI (see Alg. \ref{alg supp: general MG soft PI}) generalizes the usual PI. Instead of updating with a 1-step greedy policy it updates softly w.r.t. the 1-step greedy policy. Although this generalization has been analyzed extensively for a single-agent PI (e.g., \cite{kakade2002approximately,scherrer2014approximate}), to the best of our knowledge, it was not analyzed in the context of Markov-Games. 

By generalizing arguments from \cite{scherrer2014approximate} to framework of Zero-Sum MG  (defined in Section \ref{supp sec: pre MG}) we prove the following result.

\begin{theorem}\label{theorem: supp general soft PI Markov Games}
The sequence $v_k \eqdef v^{\mu_k,\nu_k}$ contracts toward $v^*$ with rate of $1-\eta+\gamma\eta$, i.e.,
\begin{equation*}
    {||v_k-v^*_\alpha || \leq  (1-\eta+\gamma\eta)||v_{k-1}-v^*_\alpha ||} \enspace .
\end{equation*}
\end{theorem}

As a corollary, and by plugging $\eta=1$, we get the convergence rate of Zero-Sum MG PI. Notice that although the action space is continuous the proof follows using standard machinery, since the state space is still finite. We now give the proof of the theorem.

The proof has two steps. We first show $v^*\leq v_{k+1}\leq v_{k}$, where $v_k \eqdef v^{\mu_k,\nu_k}$. Building on this fact, we prove the contraction property by generalizing technique from \cite{scherrer2014approximate}[Theorem 1], to two player game.

\begin{lemma} \label{lemma sup: inequallity between iterations soft}
$v^*\leq v_{k+1}\leq v_{k}$.
\end{lemma}
\begin{proof}
We have that $v_k = v^{\mu_k,\nu_k}$.
\begin{align}
v^{\mu_{k},\nu_{k}} &=  \bar{T}^{\nu_k} v^{\mu_{k},\nu_{k}} 
\nonumber \\
&= (1-\eta)\bar{T}^{\nu_k} v^{\mu_{k},\nu_{k}}+\eta  \bar{T}^{\nu_k} v^{\mu_{k},\nu_{k}} \nonumber \\
&\geq (1-\eta)\bar{T}^{\nu_k} v^{\mu_{k},\nu_{k}}+\min_{\nu\in P_B}\eta  \bar{T}^\nu v^{\mu_{k},\nu_{k}} \nonumber \\
&= (1-\eta)\bar{T}^{\nu_k} v^{\mu_{k},\nu_{k}}+\eta  \bar{T}^{\nu'} v^{\mu_{k},\nu_{k}} \nonumber\\
&=\max_{\mu\in P_A}\left((1-\eta)T^{\mu,\nu_k} v^{\mu_{k},\nu_{k}}\right)+  \max_{\mu\in P_A}\left(\eta\bar{T}^{\mu,\nu'} v^{\mu_{k},\nu_{k}}\right) \nonumber\\
&\geq \max_{\mu\in P_A}\left((1-\eta)T^{\mu,\nu_k} v^{\mu_{k},\nu_{k}}+  \eta\bar{T}^{\mu,\nu'} v^{\mu_{k},\nu_{k}}\right) \nonumber\\
& = \max_{\mu\in P_A}T^{\mu,(1-\eta)\nu_k+\eta \nu'} v^{\mu_{k},\nu_{k}} = \bar{T}^{(1-\eta)\nu_k+\eta \nu'} v^{\mu_{k},\nu_{k}}. \label{eq: supp inequality between iterations}
\end{align}

The first relation holds due to Proposition \ref{prop supp: fixed policy}, the forth relation holds by construction of $\nu'$, $\min_{\nu\in P_B} \bar{T}^\nu v^{\mu_k,\nu_k}=\bar{T}^{\nu'} v^{\mu_k,\nu_k}$, the fifth relation is by Definition \ref{def: supp Bellman operators MG}, the sixth relation holds since sum of maximum elements is bigger than the maximum of a sum, and the seventh relation holds since the fixed-policy Bellman operator satisfies ${T^{\mu,(1-\eta)\nu_1+\eta\nu_2} = (1-\eta)T^{\mu,\nu_1}+\eta T^{\mu,\nu_2}}$.

Due to the monotonicity of $\bar{T}^{(1-\eta)\nu_k+\eta \nu'}$ (e.g, \citet{patek1997stochastic}[Appendix A]), we can repeatedly use \eqref{eq: supp inequality between iterations},
\begin{align*}
v_k \geq \bar{T}^{(1-\eta)\nu_k+\eta \nu'}v_k \geq \cdot\cdot\cdot \geq \lim_{n\rightarrow \infty} (\bar{T}^{(1-\eta)\nu_k+\eta \nu'})^n v_k  = v_{k+1},
\end{align*}

where $ v_{k+1}= v^{\mu_{k+1},\nu_{k+1}}$. Indeed, $\bar{T}^{(1-\eta)\nu_k+\eta \nu'}$ is the optimal Bellman operator given a fixed adversary strategy, $(1-\eta)\nu_k+\eta \nu'$.

Lastly, we show that in each iteration $v^* \leq v_k$. For any adversarial strategy $\nu_k$,
\begin{align*}
v_k &= \max_{\mu\in P_A} v^{\mu,\nu_k} \geq \min_{\nu\in P_B}\max_{\mu\in P_A} v^{\mu,\nu}= v^*.
\end{align*}
Where the third relation holds by Proposition \ref{prop: supp minmax and equilibrium value}.
\end{proof}

We are now ready to prove Theorem \ref{theorem: supp general soft PI Markov Games}.

\begin{proof}
As before, define $v_k \eqdef v^{\mu_k,\nu_k}$. We have that,
\begin{align}
    v^* - v_{k+1}  & = v^* - T^{\mu_{k+1},(1-\eta)\nu_k+\eta\nu'}v_{k+1} \nonumber\\
    &\geq v^* - T^{\mu_{k+1},(1-\eta)\nu_k+\eta\nu'}v_{k} \nonumber\\
    &= (1-\eta)(v^*-  T^{\mu_{k+1},\nu_k}v_{k}) +\eta(v^*-  T^{\mu_{k+1},\nu'}v_{k}) \label{eq: supp contraction 1},
\end{align}
where the first relation holds since $v_{k+1} = v^{\mu_{k+1},(1-\eta)\nu_k+\eta \nu'}$ and the second relation holds since $T^{\mu,\nu}$ is a monotone operator and $v_{k+1}\leq v_k$ by Lemma \ref{lemma sup: inequallity between iterations soft}.

Consider the first term in \eqref{eq: supp contraction 1}.
\begin{align}
    v^*-  T^{\mu_{k+1},\nu_k}v_{k}\geq v^*-  T^{\mu_{k},\nu_k}v_{k} = v_k. \label{eq: supp contraction 2}
\end{align}
The first relation holds since $T^{\mu_{k},\nu_k}v_{k} = \max_{\mu\in P_A}T^{\mu,\nu_k}v_{k}$ and the second relation holds since by definition $v_k = v^{\mu_k,\nu_k} = T^{\mu_{k},\nu_k}v^{\mu_k,\nu_k}$ (due to Proposition \ref{prop supp: fixed policy}).

Remember that $\nu'\in \argmin_{\nu\in P_B} \bar{T}^\nu v_k$ (as in the update of Alg. \ref{alg supp: general MG soft PI}). Thus,
\begin{align}
    \bar{T}^{\nu'}v_k = \min_{\nu \in P_B}\bar{T}^{\nu}v_k =\min_{\nu \in P_B}\max_{\mu \in P_A}T^{\mu,\nu}v_k = \max_{\mu \in P_A}\min_{\nu \in P_B}T^{\mu,\nu} = T v_k, \label{eq: supp contraction intermidiate}
\end{align}
where the third relation is due to Lemma \ref{lemma: min max is max min}.

Now, for the second term in \eqref{eq: supp contraction 1} we have that,
\begin{align}
    v^*-  T^{\mu_{k+1},\nu'}v_{k}& = T v^*-  T^{\mu_{k+1},\nu'}v_{k} \nonumber\\
    &\geq  T^{\mu^*,\nu^*}v^*-  \max_{\mu\in P_A} T^{\mu,\nu'}v_{k} \nonumber\\
    &=  Tv^*-  Tv_{k}. \label{eq: supp contraction 3}
\end{align}

The first relation holds since $v^*$ is the fixed point of $T$, and the third relation holds by \eqref{eq: supp contraction intermidiate}.

Plugging \eqref{eq: supp contraction 2} and \eqref{eq: supp contraction 3} to \eqref{eq: supp contraction 1} yields,
\begin{align*}
    v^* - v_{k+1} \geq(1-\eta)(v^*-v_k) + \eta(Tv^*-  Tv_{k}).
\end{align*}
 Since $ 0 \geq v^* - v_{k+1}$ by Lemma \ref{lemma sup: inequallity between iterations soft}, we can take the max-norm and conclude the proof,
 \begin{align*}
     ||v^* - v_{k+1}||_\infty & \leq (1-\eta)||v^*-v_k ||_\infty +\eta||T v^*-T v_k ||_\infty\\
     & \leq (1-\eta)||v^*-v_k ||_\infty +\eta\gamma||v^*-v_k ||_\infty,
 \end{align*}
 where the first relation holds by the triangle inequality and the second holds since $T$ is a $\gamma$-contraction by Proposition \ref{lemma: supp everything is gamma contraction}.
\end{proof}

\section{Probabilistic Action Robust MDP}\label{supp: PR MDP}
In this section, we focus on PR-MDPs (Section~\ref{sec: probabilistic AR}) 
and map the problem of solving the optimal probabilistic robust policy to solving a Zero-Sum MG. We then continue and provide the proofs of Section~\ref{sec: probabilistic AR}, 
which are mostly corollaries to the results in Section \ref{supp: zero sum MG}. 

For simplicity, we provide the definition of PR-MDPs as given in Section~\ref{sec: probabilistic AR}.
\defPRMDP*

\subsection{Probabilistic Action Robust MDP as a Zero-Sum Markov Game} \label{supp: PR-MDP and MG}
Consider the single agent MDP on which the PR-MDP is defined, $\mathcal{M}=(\mathcal{S},\mathcal{A},P,R,\gamma)$.

\begin{assumption}\label{assumptions: single agent MDP}
\hfill
\begin{itemize}
    \item $\mathcal{A}$ is compact metric space.
    \item For any $s\in \mathcal{S}$ the reward $r$ is continuous and bounded function on $\mathcal{A}$.
    \item  For any $s\in \mathcal{S}$,   whenever $(a_n)\rightarrow(a)$, where $(a_n),(a)\in \mathcal{A}$, then $P(\cdot\mid s,a_n)$ converges weakly to $P(\cdot\mid s,a)$.
\end{itemize}
\end{assumption}

Solving the optimal probabilistic robust policy can be equivalently viewed as solving a Zero-Sum MG $\mathcal{M}_{P,\alpha}$. Let $\mathcal{M}_{P,\alpha}=(\mathcal{S},\mathcal{A},\mathcal{A},P_{P,\alpha},R_{P,\alpha},\gamma)$. Meaning, its state-space is equal to that of the original MDP, the action space of the two players is the action space of the original MDP, and its discount factor is equal to the discount factor of $\mathcal{M}$.  Its reward and dynamics are given as follows,
\begin{align}
    r_{P,\alpha}(s,a,b) = (1-\alpha)r(s,a)+\alpha r(s,b),\  P_{P,\alpha}(s'\mid s,a,b) = (1-\alpha)P(s'\mid s,a)+ \alpha P(s'\mid s,b). \label{eq: supp MG for PR MDP}
\end{align}

By Assumption \ref{assumptions: single agent MDP} on $\mathcal{M}$, Assumption \ref{assumptions: MG} on the MG is satisfied.

It is easy to prove that a value $v^{\pi^{\mathrm{mix}}_{P,\alpha}(\pi_1,\pi_2)}$ defined on $\mathcal{M}$ is equal to the value $v^{\pi_1,\pi_2}$ defined on $\mathcal{M}_{P,\alpha}$. Since there is a one-to-one correspondence between the problems, solving the later is equivalent to solving the first. 

\subsection{Proof of Proposition~\ref{prop:PR-MDP and determinstic optimal policy}}
\label{supp: proof PR-MDP and determinstic optimal policy}

Consider the Zero-Sum MG $\mathcal{M}_{P,\alpha}$, and let $P_A$ be the set of all probability measures on the Borel Sets of $\mathcal{A}$. We see that the Bellman operators of $\mathcal{M}_{P,\alpha}$ (Definition \ref{def: supp Bellman operators MG}) decouples to two terms due to \eqref{eq: supp MG for PR MDP},
\begin{align}
    T v  &= \max_{\mu\in P_A} \min_{\nu\in P_A}  r^{\mu,\nu} + \gamma P^{\mu,\nu}v \nonumber\\
        &= (1-\alpha)\left(\max_{\mu\in P_A} r^{\mu} + P^{\mu}v\right) +\alpha \left(\min_{\nu\in P_A} r^{\mu} + P^{\mu}v \right), \label{eq: supp bellman operator decouples}
\end{align}
and similarly for $T^\mu,\bar{T}^{\nu}$ and $T^{\mu,\nu}$.

According to Proposition \ref{prop: supp minmax and equilibrium value} the the optimal policy for the $\max$-agent $\mu^*$ satisfies $v^* = T v^* = T^{\mu^*} v^*$. Thus, $\mu^*$ should satisfy
\begin{align*}
    &(1-\alpha)\left(\max_{\mu\in P_A} r^{\mu} + P^{\mu}v^*\right) +\alpha \left(\min_{\nu\in P_A} r^{\mu} + P^{\mu}v^* \right) = (1-\alpha)\left(r^{\mu^*} + P^{\mu^*}v^*\right) +\alpha \left(\min_{\nu\in P_A} r^{\mu} + P^{\mu}v^* \right)\\
    \iff & \max_{\mu\in P_A} r^{\mu} + P^{\mu}v^* =  r^{\mu^*} + P^{\mu^*}v^*
\end{align*}

meaning, $\mu^* \in \max_{\mu\in P_A} r^{\mu} + P^{\mu}v^*$ which can always be solved by a deterministic policy.

\subsection{Probabilistic Action Robust and Robust MDPs}\label{supp: PR-MDP and RMDP}

Based on the mapping between a PR-MDP to a corresponding Zero-Sum MG \ref{supp: PR-MDP and MG} the relation to Robust MDPs becomes apparent. Instead for the adversary to pick an action which induces a change in the dynamics and reward \ref{eq: supp MG for PR MDP}, the adversary can directly choose the dynamics and reward. Obviously, the value of such a policy is similar under this equivalent view. We conclude the result since the adversary is defined on the class of stochastic policies $\mathcal{P}(\Pi)$.


\subsection{Proof of Proposition~\ref{prop: 1 step greedy gradients}}
\label{supp: proof  1 step greedy gradients}

Repeating the same arguments as in Policy Gradient Theorem \cite{sutton2000policy}[Theorem 1] for continuous action space we have that for any $s\in \mathcal{S}$ and $\pi\in \mathcal{P}(\Pi)$, i.e., any stochastic stationary policy,
\begin{align*}
    \nabla_\pi v^{\pi}(s) = \sum_s d^{\pi}(s)\int_{\action \in \mathcal{A}} \nabla_\pi \pi(s,\action) q^{\pi}(s,\action) d\action
\end{align*}

Notice that we can replace the integration and differentiation order by Leibniz integral rule since $\nabla_\pi v^{\pi}(s)$ exists and is bounded. Let $h(\cdot \mid s)$ be a deterministic probability measure on $A$.  Similarly to \cite{scherrer2014local} for any $s\in \mathcal{S}$,
\begin{align*}
    \langle \nabla_\pi v^{\pi}(s),h \rangle &= \sum_s d^{\pi}(s)\int_{\action \in \mathcal{A}} \langle \nabla_\pi \pi(s,\action) ,h \rangle  q^{\pi}(s,\action) d\action\\
     & = \sum_s d^{\pi}(s)q^{\pi}(s,h(s)).
\end{align*}

To minimize $\langle \nabla_\pi v^{\pi}(s),h \rangle$ we choose for any $s\in \mathcal{S}$, $\action_h \in \argmin_{a} q^{\pi}(\cdot,a) = \argmin_{\pi'} r^{\pi'}+\gamma P^{\pi'} v^\pi$.

\subsection{Proof of Theorem~\ref{theorem: PR-MDP contraction}}
\label{supp: PR-MDP contraction}
The theorem is a corollary of Theorem \ref{theorem: supp general soft PI Markov Games} and Proposition~\ref{prop: 1 step greedy gradients}, 
while using the structure of the defined zero-sum MG for PR-MDP in Section \ref{supp: PR-MDP and MG}, $\mathcal{M}_{P,\alpha}$.

Specifically, the first stage of the general Soft Zero-Sum MG PI \ref{alg supp: general MG soft PI} is similar to the first stage of Soft Probabilistic Robust PI~\ref{alg: soft alpha robust PI}.
Furthermore, for $\mathcal{M}_{P,\alpha}$ it holds for any bounded $v\in \mathbb{R}^{|\mathcal{S}|}$,
\begin{align*}
    \argmin_{\nu\in P_A} \bar{T}^\nu v &=\argmin_{\nu\in P_A} \max_{\mu\in P_A} T^{\mu,\nu}v \\
    &=\argmin_{\nu\in P_A} \max_{\mu\in P_A} (1-\alpha)(r^\mu+\gamma P^\mu v)+\alpha(r^\nu+\gamma P^\nu v)\\
    &=\argmin_{\nu\in P_A} \left(r^\nu+\gamma P^\nu v\right),
\end{align*}
where the first relation holds by definition \ref{def: supp Bellman operators MG}, the second relation holds due to the specific form of the Bellman operators similarly to \eqref{eq: supp bellman operator decouples}, and the third relation holds since the first term does not depend on $\nu$.

By using Proposition~\ref{prop: 1 step greedy gradients}
we get that Soft Probabilistic Robust PI~\ref{alg: soft alpha robust PI}
is an instance of the more general Soft Zero-Sum MG PI \ref{alg supp: general MG soft PI}, and prove the Theorem as a corollary of Theorem \ref{theorem: supp general soft PI Markov Games}.

\section{Noisy Action Robust MDP as a Zero-Sum Markov Game}\label{supp: NR MDP}

We focus on NR-MDPs (Section~\ref{sec: noisy AR}) 
and map the problem of solving the optimal noisy robust policy to solving a Zero-Sum MG. As in previous section, the proofs of Section~\ref{sec: noisy AR}, 
are mostly corollaries to the results in Section \ref{supp: zero sum MG}. 

For simplicity, we provide the definition of NR-MDPs as given in Section~\ref{sec: probabilistic AR}.
\setcounter{defn}{1}

\defNRMDP*

\subsection{Noisy Action Robust MDP as a Zero-Sum Markov Game} \label{supp: NR-MDP and MG}
Consider the single agent MDP on which the NR-MDP is defined, $\mathcal{M}=(\mathcal{S},\mathcal{A},P,R,\gamma)$ and assume it satisfies Assumption \ref{assumptions: single agent MDP}. Solving the optimal probabilistic robust policy can be equivalently viewed as solving a Zero-Sum MG $\mathcal{M}_{N,\alpha}$. Let $\mathcal{M}_{N,\alpha}=(\mathcal{S},\mathcal{A},\mathcal{A},P_{N,\alpha},R_{N,\alpha},\gamma)$. Meaning, its state-space is equal to that of the original MDP, the action space of the two players is the action space of the original MDP, and its discount factor is equal to the discount factor of $\mathcal{M}$.  Its reward and dynamics are given as follows,
\begin{align}
    r_{N,\alpha}(s,a,b) = r(s,(1-\alpha)a+\alpha b),\  P_{P,\alpha}(s'\mid s,a,b) = P(s'\mid s,(1-\alpha)a+\alpha b). \label{eq: supp MG for NR MDP}
\end{align}

Since the single agent MDP satisfies Assumption Assumption \ref{assumptions: single agent MDP}, the MG game defined by $\mathcal{M}_{N,\alpha}$ satisfies \ref{assumptions: MG}.

It is easy to prove that a value $v^{\pi_N^{\mathrm{mix}}(\pi_1,\pi_2)}$ defined on the induced NR-MDP from $\mathcal{M}$ is equal to the value $v^{\pi_1,\pi_2}$ defined on the MG $\mathcal{M}_{N,\alpha}$. Since there is a one-to-one correspondence between the problems, solving the later is equivalent to solving the first. 

\subsection{Proof of Proposition~\ref{prop: noisy policy is stochastic}}
\label{supp: noisy policy is stochastic}
Consider an MDP with a single state a quadratic reward of the form $r(a) = a^2$ where $a\in [-1,1]$. In this case, the solution does not depend on the horizon and an optimal action w.r.t. a single time step will be the solution for the discounted reward. Denote $\mathcal{P}([-1,1])$ as the set of all probability measures on the Borel sets of $[-1,1]$.

If both of the players are only allowed to take deterministic actions, then the min-max and max-min values are not equivalent,
\begin{align*}
    &\max_{a\in [-1,1]}\min_{b\in [-1,1]} ((1-\alpha)a+\alpha b)^2 =  \begin{cases}
    (1-2\alpha)^2,\ \alpha\leq 0.5 \\
    0,\ \alpha>0.5 \\
    \end{cases}\\
    &\min_{b\in [-1,1]} \max_{a\in [-1,1]} ((1-\alpha)a+\alpha b)^2 = (1-\alpha)^2.
\end{align*}

Thus, for this example, strong duality on the sets of deterministic policies does not hold, $$ \max_{a\in [-1,1]}\min_{b\in [-1,1]} ((1-\alpha)a+\alpha b)^2< \min_{b\in [-1,1]} \max_{a\in [-1,1]} ((1-\alpha)a+\alpha b)^2 = (1-\alpha)^2.$$

Furthermore, we now show that considering random policies can increase the value. Let the policy of the $\max$-player be $P(a=-1)=P(a=1)=0.5$, obviously, $P\in \mathcal{P}([-1,1])$. For this policy, we have that,
\begin{align*}
    \min_{b\in [-1,1]} \mathbb{E}_{a\sim P(\cdot)}[(1-\alpha)a+\alpha b)^2] = \min_{b\in [-1,1]} (1-\alpha)^2+\alpha^2 b = (1-\alpha)^2.
\end{align*}

We conclude that for this example $$ \max_{a\in [-1,1]}\min_{b\in [-1,1]} ((1-\alpha)a+\alpha b)^2 < \max_{P\in \mathcal{P}([-1,1])}\min_{b\in [-1,1]} \mathbb{E}_{a\sim P}[((1-\alpha)a+\alpha b)^2].$$

\subsection{Policy Iteration of NR-MDP} \label{supp: PI for NR-MDP}
We can use the Soft Zero-Sum MG PI (see Algorithm \ref{alg supp: general MG soft PI}), or, by fixing $\eta=1$, Zero-Sum MG PI. 

The algorithm repeats two stages of (i) solving an MDP by fixing the adversary policy, (ii) solving a 1-step greedy minimax decision problem on the set of stochastic policies. This comes in contrast to the corresponding PI algorithm that solves PR-MDP. There, stage (ii) involved in solving a \emph{single} agent, 1-step greedy, decision problem. This problem can be more easily solved by function maximization. 

Furthermore, this fact suggest that a simple Frank-Wolfe update \cite{frank1956algorithm}, as was performed in Soft Probabilistic Robust PI (Algorithm~\ref{alg: soft alpha robust PI}) 
would not work, at least not using the analysis we suggested here. Meaning, a relation between the maximal projection on the gradient $\nabla_\pi v^\pi$ and the 1-step greedy minimax decision problem, as shown to hold in Proposition~\ref{prop: 1 step greedy gradients}, 
would not exists.

\section{Actor Gradients Proof}\label{apndx: gradients proof}
\begin{proof}
    Our proof follows the proof of the deterministic policy gradients (DPG) \cite{silver2014deterministic}.
    
    In order to retain consistency with \cite{silver2014deterministic}, we denote the deterministic policy $\pi$ by $\mu: S \mapsto A$. The parametrized policies $\mu_\theta$ and $\bar \mu_{\bar \theta}$ are, respectively, the actor and adversary policies. We refer to the $\alpha$-mixture policy $\pi^{\mathrm{mix}}_{N/P,\alpha} (\mu_\theta, \bar \mu_{\bar \theta})$ simply as $\pi^{\mathrm{mix}}_{N/P,\alpha} (\theta, \bar \theta)$, for ease of notation.
    
    \begin{assumption}
        $p(\state' \mid \state, \action), \nabla_{\action} p(\state' \mid \state, \action), \mu_\theta (\state), \nabla_\theta \mu_\theta (\state), \bar \mu_{\bar \theta} (\state), \nabla_{\bar \theta} \bar \mu_{\bar \theta} (\state), r(\state, \action), \nabla_a r (\state, \action), p_1(\state)$ are continuous in all parameters and variables $\state, \action, \state'$ and $x$.
    \end{assumption}
    
    \begin{assumption}
        There exists a $b$ and $L$ such that $\sup_{\state} p_1(\state) < b, \sup_{\action, \state, \state'} p(\state' \mid \state, \action) < b, \sup_{\action, \state} r(\state, \action) < b, \sup_{\action, \state, \state'} || \nabla_{\action} p(\state' \mid \state, \action) || < L,$ and $\sup_{\state, \action} ||\nabla_{\action} r(\state, \action)|| < L$.
    \end{assumption}
    
    
    \paragraph{NR-MDP:}
    \begin{align*}
        \nabla_\theta v^{\pi^{\mathrm{mix}}_{N,\alpha} (\theta, \bar \theta)} &= \nabla_\theta Q^{\pi^{\mathrm{mix}}_{N,\alpha} (\theta, \bar \theta)} (\state, \pi^{\mathrm{mix}}_{N,\alpha} (\theta, \bar \theta) (\state)) \\
        &= \nabla_\theta \left( r(\state, \pi^{\mathrm{mix}}_{N,\alpha} (\theta, \bar \theta) (\state)) + \int_S \gamma p(\state' \mid \state, \pi^{\mathrm{mix}}_{N,\alpha} (\theta, \bar \theta) (\state)) v^{\pi^{\mathrm{mix}}_{N,\alpha} (\theta, \bar \theta)} (\state') \right) d\state' \\
        &= \nabla_\theta \pi^{\mathrm{mix}}_{N,\alpha} (\theta, \bar \theta) (\state) \nabla_{\action} r(\state, \action) \mid_{\action = \pi^{\mathrm{mix}}_{N,\alpha} (\theta, \bar \theta) (\state)} + \nabla_\theta \int_S \gamma p(\state' \mid \state, \pi^{\mathrm{mix}}_{N,\alpha} (\theta, \bar \theta)(\state)) v^{\pi^{\mathrm{mix}}_{N,\alpha} (\theta, \bar \theta)} (\state') d\state' \\
        &= \nabla_\theta \pi^{\mathrm{mix}}_{N,\alpha} (\theta, \bar \theta) (\state) \nabla_{\theta} r(\state, \action) \mid_{\action = \pi^{\mathrm{mix}}_{N,\alpha} (\theta, \bar \theta) (\state)} \\
        &\enspace + \int_S \gamma \left( p(\state' \mid \state, \pi^{\mathrm{mix}}_{N,\alpha} (\theta, \bar \theta) (\state)) \nabla_\theta v^{\pi^{\mathrm{mix}}_{N,\alpha} (\theta, \bar \theta)} (\state') + \nabla_{\theta} \pi^{\mathrm{mix}}_{N,\alpha} (\theta, \bar \theta) (\state) \nabla_{\action} p(\state' \mid \state, \action) \mid_{\action = \pi^{\mathrm{mix}}_{N,\alpha} (\theta, \bar \theta) (\state)} v^{\pi^{\mathrm{mix}}_{N,\alpha} (\theta, \bar \theta)} (\state') \right)d\state' \\
        &= \nabla_\theta \pi^{\mathrm{mix}}_{N,\alpha} (\theta, \bar \theta) (\state) \nabla_{\action} \left( r(\state, \action) + \int_S \gamma p(\state' \mid \state, \action) v^{\pi^{\mathrm{mix}}_{N,\alpha} (\theta, \bar \theta)} (\state') d\state' \right) \mid_{\action = \pi^{\mathrm{mix}}_{N,\alpha} (\theta, \bar \theta) (\state)} \\
        &\enspace + \int_S \gamma p(\state' \mid \state, \pi^{\mathrm{mix}}_{N,\alpha} (\theta, \bar \theta) (\state)) \nabla_\theta v^{\pi^{\mathrm{mix}}_{N,\alpha} (\theta, \bar \theta)} (\state') d\state' \\
        &= \nabla_\theta \pi^{\mathrm{mix}}_{N,\alpha} (\theta, \bar \theta) (\state) \nabla_{\action} Q^{\pi^{\mathrm{mix}}_{N,\alpha} (\theta, \bar \theta)} (\state, \action) \mid_{\action = \pi^{\mathrm{mix}}_{N,\alpha} (\theta, \bar \theta) (\state)} + \int_S \gamma p(\state \rightarrow \state', 1, \pi^{\mathrm{mix}}_{N,\alpha} (\theta, \bar \theta)) \nabla_\theta v^{\pi^{\mathrm{mix}}_{N,\alpha} (\theta, \bar \theta)} (\state') d\state' \enspace .
    \end{align*}
    Where $p(\state \rightarrow \state', t, \pi)$ denotes the density at state $\state'$ after transitioning for $t$ steps from state $\state$. Iterating this formula leads to the following result:
    \begin{align*}
        \nabla_\theta v^{\pi^{\mathrm{mix}}_{N,\alpha} (\theta, \bar \theta)} &= \nabla_\theta \pi^{\mathrm{mix}}_{N,\alpha} (\theta, \bar \theta) (\state) \nabla_{\action} Q^{\pi^{\mathrm{mix}}_{N,\alpha} (\theta, \bar \theta)} (\state, \action) \mid_{\action = \pi^{\mathrm{mix}}_{N,\alpha} (\theta, \bar \theta) (\state)} \\
        &\enspace + \int_S \gamma p(\state \rightarrow \state', 1, \pi^{\mathrm{mix}}_{N,\alpha} (\theta, \bar \theta)) \nabla_\theta \pi^{\mathrm{mix}}_{N,\alpha} (\theta, \bar \theta) (\state') \nabla_{\action} Q^{\pi^{\mathrm{mix}}_{N,\alpha} (\theta, \bar \theta)} (\state', \action) \mid_{\action=\pi^{\mathrm{mix}}_{N,\alpha} (\theta, \bar \theta)(\state')} d\state' \\
        &\enspace + \int_S \gamma p(\state \rightarrow \state', 1, \pi^{\mathrm{mix}}_{N,\alpha} (\theta, \bar \theta)) \int_S \gamma p(\state' \rightarrow \state'', 1, \pi^{\mathrm{mix}}_{N,\alpha} (\theta, \bar \theta)) \nabla_\theta v^{\pi^{\mathrm{mix}}_{N,\alpha} (\theta, \bar \theta)} (\state'') d\state'' d\state' \\
        &= \nabla_\theta \pi^{\mathrm{mix}}_{N,\alpha} (\theta, \bar \theta) (\state) \nabla_{\action} Q^{\pi^{\mathrm{mix}}_{N,\alpha} (\theta, \bar \theta)} (\state, \action) \mid_{\action = \pi^{\mathrm{mix}}_{N,\alpha} (\theta, \bar \theta) (\state)} \\
        &\enspace + \int_S \gamma p(\state \rightarrow \state', 1, \pi^{\mathrm{mix}}_{N,\alpha} (\theta, \bar \theta)) \nabla_\theta \pi^{\mathrm{mix}}_{N,\alpha} (\theta, \bar \theta) (\state') \nabla_{\action} Q^{\pi^{\mathrm{mix}}_{N,\alpha} (\theta, \bar \theta)} (\state', \action) \mid_{\action=\pi^{\mathrm{mix}}_{N,\alpha} (\theta, \bar \theta)(\state')} d\state' \\
        &\enspace + \int_S \gamma^2 p(\state \rightarrow \state', 2, \pi^{\mathrm{mix}}_{N,\alpha} (\theta, \bar \theta)) \nabla_\theta v^{\pi^{\mathrm{mix}}_{N,\alpha} (\theta, \bar \theta)} (\state') d\state' \\
        &= \int_S \sum_{t=0}^\infty \gamma^t p(\state \rightarrow s', t, \pi^{\mathrm{mix}}_{N,\alpha} (\theta, \bar \theta)) \nabla_\theta \pi^{\mathrm{mix}}_{N,\alpha} (\theta, \bar \theta) (\state') \nabla_{\action} Q^{\pi^{\mathrm{mix}}_{N,\alpha} (\theta, \bar \theta)} (\state', \action) \mid_{\action = \pi^{\mathrm{mix}}_{N,\alpha} (\theta, \bar \theta) (\state')} d\state' \enspace .
    \end{align*}
    Taking the expectation over $S_1$:
    \begin{align*}
        \nabla_\theta J(\pi^{\mathrm{mix}}_{N,\alpha} (\theta, \bar \theta)) &= \nabla_\theta \int_S p_1 (\state) v^{\pi^{\mathrm{mix}}_{N,\alpha} (\theta, \bar \theta)} (\state) d\state \\
        &= \int_S p_1 (\state) \nabla_\theta v^{\pi^{\mathrm{mix}}_{N,\alpha} (\theta, \bar \theta)} (\state) d\state \\
        &= \int_S \int_S \sum_{t = 0}^\infty \gamma^t p_1 (\state) p(\state \rightarrow \state', t, \pi^{\mathrm{mix}}_{N,\alpha} (\theta, \bar \theta)) \nabla_\theta \pi^{\mathrm{mix}}_{N,\alpha} (\theta, \bar \theta) (\state') \nabla_{\action} Q^{\pi^{\mathrm{mix}}_{N,\alpha} (\theta, \bar \theta)} (\state', \action) \mid_{\action = \pi^{\mathrm{mix}}_{N,\alpha} (\theta, \bar \theta) (\state')} d\state' d\state \\
        &= \int_S \rho^{\pi^{\mathrm{mix}}_{N,\alpha} (\theta, \bar \theta)} \nabla_\theta \pi^{\mathrm{mix}}_{N,\alpha} (\theta, \bar \theta) (\state) \nabla_{\action} Q^{\pi^{\mathrm{mix}}_{N,\alpha} (\theta, \bar \theta)} (\state, \action) \mid_{\action = \pi^{\mathrm{mix}}_{N,\alpha} (\theta, \bar \theta) (\state)} d\state \\
        &= \int_S \rho^{\pi^{\mathrm{mix}}_{N,\alpha} (\theta, \bar \theta)} \nabla_\theta ((1 - \alpha) \mu_\theta (\state) + \alpha \bar \mu_{\bar \theta} (\state)) \nabla_{\action} Q^{\pi^{\mathrm{mix}}_{N,\alpha} (\theta, \bar \theta)} (\state, \action) \mid_{\action = \pi^{\mathrm{mix}}_{N,\alpha} (\theta, \bar \theta) (\state)} d\state \\
        &= (1 - \alpha) \int_S \rho^{\pi^{\mathrm{mix}}_{N,\alpha} (\theta, \bar \theta)} \nabla_\theta \mu_\theta (\state) \nabla_{\action} Q^{\pi^{\mathrm{mix}}_{N,\alpha} (\theta, \bar \theta)} (\state, \action) \mid_{\action = \pi^{\mathrm{mix}}_{N,\alpha} (\theta, \bar \theta) (\state)} d\state
    \end{align*}
    notice that compared to the standard DPGs \cite{silver2014deterministic}, the gradient is w.r.t. the actor's (adversary's) policy and is weighted by $1 - \alpha$ ($\alpha$). Similar to the DPG, the gradient of the action-value function is taken w.r.t. the action taken (the mixture policy).
    
    \paragraph{PR-MDP:}
    The PR-MDP, constructed by two deterministic policies $\mu_\theta$ and $\bar \mu_{\bar \theta}$ can be defined as follows:
    \begin{equation*}
        \pi^{\mathrm{mix}}_{P,\alpha}(u \mid s; \theta,\bar \theta) = (1-\alpha)\delta(u-\mu_\theta(s))+\alpha\delta(u - \bar \mu_{\bar \theta}(s)) .
    \end{equation*}
    
    \begin{equation*}
        v^{\pi^{\mathrm{mix}}_{P,\alpha} (\theta, \bar \theta)} = \int_A \pi^{\mathrm{mix}}_{P,\alpha}(u \mid s; \theta,\bar \theta) Q^{\pi^{\mathrm{mix}}_{P,\alpha} (\theta, \bar \theta)} (\state, \pi^{\mathrm{mix}}_{P,\alpha} (\theta, \bar \theta) (\state)) du
    \end{equation*}
    
    \begin{align*}
        \nabla_\theta v^{\pi^{\mathrm{mix}}_{P,\alpha} (\theta, \bar \theta)} &= \nabla_\theta \int_A \pi^{\mathrm{mix}}_{P,\alpha}(u \mid s; \theta,\bar \theta) Q^{\pi^{\mathrm{mix}}_{P,\alpha} (\theta, \bar \theta)} (\state, u) du \\
        &= \nabla_\theta [(1 - \alpha) Q^{\pi^{\mathrm{mix}}_{P,\alpha} (\theta, \bar \theta)} (\state, \mu_\theta (\state)) + \alpha Q^{\pi^{\mathrm{mix}}_{P,\alpha} (\theta, \bar \theta)} (\state, \bar \mu_{\bar \theta} (\state))] \\
        &= (1 - \alpha) \nabla_\theta Q^{\pi^{\mathrm{mix}}_{P,\alpha} (\theta, \bar \theta)} (\state, \mu_\theta (\state)) + \alpha \nabla_\theta Q^{\pi^{\mathrm{mix}}_{P,\alpha} (\theta, \bar \theta)} (\state, \bar \mu_{\bar \theta} (\state))
    \end{align*}
    
    we address each element, (1) $\nabla_\theta Q^{\pi^{\mathrm{mix}}_{P,\alpha} (\theta, \bar \theta)} (\state, \mu_\theta (\state))$ and (2) $Q^{\pi^{\mathrm{mix}}_{P,\alpha} (\theta, \bar \theta)} (\state, \bar \mu_{\bar \theta} (\state))$, individually:
    
    (1):
    \begin{align*}
        \nabla_\theta Q^{\pi^{\mathrm{mix}}_{P,\alpha} (\theta, \bar \theta)} (\state, \mu_\theta (\state)) &= \nabla \left( r(\state, \mu_\theta (\state)) + \int_S \gamma p(\state' \mid \state, \mu_\theta (\state)) v^{\pi^{\mathrm{mix}}_{P,\alpha} (\theta, \bar \theta)} (\state') \right) d\state' \\
        &= \nabla_\theta \mu_\theta (\state) \nabla_{\action} r(\state, \action) \mid_{\action = \mu_\theta (\state)} + \nabla_\theta \int_S \gamma p(\state' \mid \state, \mu_\theta (\state)) v^{\pi^{\mathrm{mix}}_{P,\alpha} (\theta, \bar \theta)} (\state') d\state' \\
        &= \nabla_\theta \mu_\theta (\state) \nabla_{\theta} r(\state, \action) \mid_{\action = \mu_\theta (\state)} \\
        &\enspace + \int_S \gamma \left( p(\state' \mid \state, \mu_\theta (\state)) \nabla_\theta v^{\pi^{\mathrm{mix}}_{P,\alpha} (\theta, \bar \theta)} (\state') + \nabla_{\theta} \mu_\theta (\state) \nabla_{\action} p(\state' \mid \state, \action) \mid_{\action = \mu_\theta (\state)} v^{\pi^{\mathrm{mix}}_{P,\alpha} (\theta, \bar \theta)} (\state') \right)d\state' \\
        &= \nabla_\theta \mu_\theta (\state) \nabla_{\action} \left( r(\state, \action) + \int_S \gamma p(\state' \mid \state, \action) v^{\pi^{\mathrm{mix}}_{P,\alpha} (\theta, \bar \theta)} (\state') d\state' \right) \mid_{\action = \mu_\theta (\state)} \\
        &\enspace + \int_S \gamma p(\state' \mid \state, \mu_\theta (\state)) \nabla_\theta v^{\pi^{\mathrm{mix}}_{P,\alpha} (\theta, \bar \theta)} (\state') d\state' \\
        &= \nabla_\theta \mu_\theta (\state) \nabla_{\action} Q^{\pi^{\mathrm{mix}}_{P,\alpha} (\theta, \bar \theta)} (\state, \action) \mid_{\action = \mu_\theta (\state)} + \int_S \gamma p(\state \rightarrow \state', 1, \mu_\theta) \nabla_\theta v^{\pi^{\mathrm{mix}}_{P,\alpha} (\theta, \bar \theta)} (\state') d\state' \enspace .
    \end{align*}
    Where $p(\state \rightarrow \state', t, \pi)$ denotes the density at state $\state'$ after transitioning for $t$ steps from state $\state$.
    
    (2):
    \begin{align*}
        \nabla_\theta Q^{\pi^{\mathrm{mix}}_{P,\alpha} (\theta, \bar \theta)} (\state, \bar\mu_{\bar\theta} (\state)) &= \nabla_\theta \left( r(\state, \bar\mu_{\bar\theta} (\state)) + \int_S \gamma p(\state' \mid \state, \bar\mu_{\bar\theta} (\state)) v^{\pi^{\mathrm{mix}}_{P,\alpha} (\theta, \bar \theta)} (\state') \right) d\state' \\
        &= \nabla_\theta \bar\mu_{\bar\theta} (\state) \nabla_{\action} r(\state, \action) \mid_{\action = \bar\mu_{\bar\theta} (\state)} + \nabla_\theta \int_S \gamma p(\state' \mid \state, \bar\mu_{\bar\theta} (\state)) v^{\pi^{\mathrm{mix}}_{P,\alpha} (\theta, \bar \theta)} (\state') d\state' \\
        &= \int_S \gamma \left( p(\state' \mid \state, \bar\mu_{\bar\theta} (\state)) \nabla_\theta v^{\pi^{\mathrm{mix}}_{P,\alpha} (\theta, \bar \theta)} (\state') + \nabla_{\theta} \bar\mu_{\bar\theta} (\state) \nabla_{\action} p(\state' \mid \state, \action) \mid_{\action = \bar\mu_{\bar\theta} (\state)} v^{\pi^{\mathrm{mix}}_{P,\alpha} (\theta, \bar \theta)} (\state') \right)d\state' \\
        &= \int_S \gamma p(\state' \mid \state, \bar\mu_{\bar\theta} (\state)) \nabla_\theta v^{\pi^{\mathrm{mix}}_{P,\alpha} (\theta, \bar \theta)} (\state') d\state' \enspace .
    \end{align*}

    Hence:
    \begin{align*}
        \nabla_\theta v^{\pi^{\mathrm{mix}}_{P,\alpha} (\theta, \bar \theta)} &= (1 - \alpha) \nabla_\theta Q^{\pi^{\mathrm{mix}}_{P,\alpha} (\theta, \bar \theta)} (\state, \mu_\theta (\state)) + \alpha \nabla_\theta Q^{\pi^{\mathrm{mix}}_{P,\alpha} (\theta, \bar \theta)} (\state, \bar \mu_{\bar \theta} (\state)) \\
        &= (1-\alpha) \nabla_\theta \mu_\theta (\state) \nabla_{\action} Q^{\pi^{\mathrm{mix}}_{P,\alpha} (\theta, \bar \theta)} (\state, \action) \mid_{\action = \mu_\theta (\state)} \\
        &+ (1-\alpha) \int_S \gamma p(\state \rightarrow \state', 1, \mu_\theta) \nabla_\theta v^{\pi^{\mathrm{mix}}_{P,\alpha} (\theta, \bar \theta)} (\state') d\state' + \alpha \int_S \gamma p(\state' \mid \state, \bar\mu_{\bar\theta} (\state)) \nabla_\theta v^{\pi^{\mathrm{mix}}_{P,\alpha} (\theta, \bar \theta)} (\state') d\state' \\
        &= (1-\alpha) \nabla_\theta \mu_\theta (\state) \nabla_{\action} Q^{\pi^{\mathrm{mix}}_{P,\alpha} (\theta, \bar \theta)} (\state, \action) \mid_{\action = \mu_\theta (\state)} + \int_S \gamma p(\state' \mid \state, \pi^{\mathrm{mix}}_{P,\alpha}(\theta, \bar\theta) (\state)) \nabla_\theta v^{\pi^{\mathrm{mix}}_{P,\alpha} (\theta, \bar \theta)} (\state') d\state'
    \end{align*}
    
    Applying this iteratively:
    
    \begin{align*}
        \nabla_\theta v^{\pi^{\mathrm{mix}}_{P,\alpha} (\theta, \bar \theta)} &= (1-\alpha) \nabla_\theta \mu_\theta (\state) \nabla_{\action} Q^{\pi^{\mathrm{mix}}_{P,\alpha} (\theta, \bar \theta)} (\state, \action) \mid_{\action = \mu_\theta (\state)} \\
        &+ \int_S \gamma p(\state' \mid \state, \pi^{\mathrm{mix}}_{P,\alpha}(\theta, \bar\theta) (\state)) \nabla_\theta v^{\pi^{\mathrm{mix}}_{P,\alpha} (\theta, \bar \theta)} (\state') d\state' \\
        &= (1-\alpha) \nabla_\theta \mu_\theta (\state) \nabla_{\action} Q^{\pi^{\mathrm{mix}}_{P,\alpha} (\theta, \bar \theta)} (\state, \action) \mid_{\action = \mu_\theta (\state)} \\
        &\enspace + \int_S \gamma p(\state \rightarrow \state', 1, \pi^{\mathrm{mix}}_{P,\alpha} (\theta, \bar \theta)) \nabla_\theta \mu_\theta (\state') \nabla_{\action} Q^{\pi^{\mathrm{mix}}_{N,\alpha} (\theta, \bar \theta)} (\state', \action) \mid_{\action=\mu_\theta(\state')} d\state' \\
        &\enspace + \int_S \gamma p(\state \rightarrow \state', 1, \pi^{\mathrm{mix}}_{P,\alpha} (\theta, \bar \theta)) \int_S \gamma p(\state' \rightarrow \state'', 1, \pi^{\mathrm{mix}}_{P,\alpha} (\theta, \bar \theta)) \nabla_\theta v^{\pi^{\mathrm{mix}}_{N,\alpha} (\theta, \bar \theta)} (\state'') d\state'' d\state' \\
        &= (1-\alpha) \nabla_\theta \mu_\theta (\state) \nabla_{\action} Q^{\pi^{\mathrm{mix}}_{P,\alpha} (\theta, \bar \theta)} (\state, \action) \mid_{\action = \mu_\theta (\state)} \\
        &\enspace + (1-\alpha)\int_S \gamma p(\state \rightarrow \state', 1, \pi^{\mathrm{mix}}_{P,\alpha} (\theta, \bar \theta)) \nabla_\theta \mu_\theta (\state') \nabla_{\action} Q^{\pi^{\mathrm{mix}}_{P,\alpha} (\theta, \bar \theta)} (\state', \action) \mid_{\action=\mu_\theta(\state')} d\state' \\
        &\enspace + \int_S \gamma^2 p(\state \rightarrow \state', 2, \pi^{\mathrm{mix}}_{P,\alpha} (\theta, \bar \theta)) \nabla_\theta v^{\pi^{\mathrm{mix}}_{P,\alpha} (\theta, \bar \theta)} (\state') d\state' \\
        &= (1-\alpha) \int_S \sum_{t=0}^\infty \gamma^t p(\state \rightarrow s', t, \pi^{\mathrm{mix}}_{P,\alpha} (\theta, \bar \theta)) \nabla_\theta \mu_\theta (\state') \nabla_{\action} Q^{\pi^{\mathrm{mix}}_{P,\alpha} (\theta, \bar \theta)} (\state', \action) \mid_{\action = \mu\theta(\state')} d\state' \enspace .
    \end{align*}
    
    Taking the expectation over $S_1$:
    \begin{align*}
        \nabla_\theta J(\pi^{\mathrm{mix}}_{P,\alpha} (\theta, \bar \theta)) &= \nabla_\theta \int_S p_1 (\state) v^{\pi^{\mathrm{mix}}_{P,\alpha} (\theta, \bar \theta)} (\state) d\state \\
        &= \int_S p_1 (\state) \nabla_\theta v^{\pi^{\mathrm{mix}}_{P,\alpha} (\theta, \bar \theta)} (\state) d\state \\
        &= (1-\alpha) \int_S \int_S \sum_{t = 0}^\infty \gamma^t p_1 (\state) p(\state \rightarrow \state', t, \pi^{\mathrm{mix}}_{P,\alpha} (\theta, \bar \theta)) \nabla_\theta \pi_\theta (\state') \nabla_{\action} Q^{\pi^{\mathrm{mix}}_{P,\alpha} (\theta, \bar \theta)} (\state', \action) \mid_{\action = \pi_\theta (\state')} d\state' d\state \\
        &= (1-\alpha) \int_S \rho^{\pi^{\mathrm{mix}}_{P,\alpha} (\theta, \bar \theta)} \nabla_\theta \pi_\theta (\state) \nabla_{\action} Q^{\pi^{\mathrm{mix}}_{P,\alpha} (\theta, \bar \theta)} (\state, \action) \mid_{\action = \pi_\theta (\state)} d\state
    \end{align*}
    
    the resulting gradient update for the actor does not directly take into consideration the policy of the adversary, thus resulting in a gradient rule similar (weighted by $(1 - \alpha)$ for the actor and $\alpha$ for the adversary) to that seen in \citet{silver2014deterministic}.
    
    Intuitively, as the action is sampled w.p. $(1 - \alpha)$ from the actor and w.p, $\alpha$ from the adversary, each player acts greedily at the immediate step ignoring potential perturbations. The mutual effect of the actor and adversary is attained through the $Q$ value which captures the long term return of the mixture policy.
\end{proof}
\newpage

\begin{algorithm}[H]
\caption{Action-Robust DDPG}\label{alg:robust_ddpg}
\begin{algorithmic}
    \STATE \textbf{Input:} Actor update steps ($N$), uncertainty value $\alpha$ and discount factor $\gamma$
    \STATE Randomly initialize critic network $Q(\state, \action; \phi)$, actor $f(\state; \theta)$ and adversary $\bar f (\state; \bar \theta)$
    \STATE Initialize target networks with weights $\phi^-, \theta^-, \bar \theta^-$
    \STATE Initialize replay buffer $R$
    
    \FOR{episode in $0...M$}
        \STATE Receive initial state $\state_0$
        \FOR{t in $0...T$}
            \STATE Sample action $\action_t = \begin{cases} f(\state; \theta_\pi) \text{ w.p. } (1 - \alpha) \text{ and } \bar f(s; \theta_{\bar \pi}) \text{ otherwise } & \text{, PR-MDP} \\
            (1 - \alpha) f(\state; \theta_\pi) + \alpha \bar f(\state; \bar \theta_{\bar \pi}) & \text{, NR-MDP}
            \end{cases}$
            \STATE $\tilde \action_t = \action_t$ + exploration noise
            \STATE Execute action $\tilde \action_t$ and observe reward $r_t$ and new state $s_{t+1}$
            \STATE Store transition $(\state_t, \tilde \action_t, r_t, \state_{t+1})$ in $R$
            
            \FOR{i in $0...N$}
                \STATE Sample batch from replay buffer
                \STATE Update actor:
                \STATE $\theta \gets \begin{cases} \nabla_\theta (1 - \alpha) Q(\state, f(\state; \theta)) &, \text{PR-MDP} \\
                \nabla_\theta Q(\state, (1 - \alpha) f(\state; \theta) + \alpha \bar f (\state; \bar \theta)) &, \text{NR-MDP}
                \end{cases}$
                \STATE Update critic:
                \STATE $\phi \gets \begin{cases}\nabla_\phi || r + \gamma [(1-\alpha) Q(\state', f(\state'; \theta^-)) + \alpha Q(\state', f(\state'; \bar \theta^-))] ||_2^2  & , \text{PR-MDP}\\
                \nabla_\phi || r + \gamma [Q(\state', (1-\alpha) f(\state'; \theta^-) + \alpha f(\state'; \bar \theta^-))] ||_2^2  & , \text{NR-MDP}
                \end{cases}$
            \ENDFOR
            
            \STATE Sample batch from replay buffer
            \STATE Update adversary:
            \STATE $\bar \theta \gets \begin{cases} \nabla_{\bar \theta} \alpha Q(\state, \bar f(\state; \bar \theta)) &, \text{PR-MDP} \\
                \nabla_{\bar \theta} Q(\state, (1 - \alpha) f(\state; \theta) + \alpha \bar f (\state; \bar \theta)) &, \text{NR-MDP}
                \end{cases}$
            \STATE Update critic
            \STATE Update the target networks:
            \STATE $\enspace\enspace\enspace\enspace\enspace\enspace \theta^- \gets \tau \theta + (1 - \tau) \theta^-$
            \STATE $\enspace\enspace\enspace\enspace\enspace\enspace \bar \theta^- \gets \tau \bar \theta + (1 - \tau) \bar \theta^-$
            \STATE $\enspace\enspace\enspace\enspace\enspace\enspace \phi^- \gets \tau \phi + (1 - \tau) \phi^-$
        \ENDFOR
    \ENDFOR
\end{algorithmic}
\end{algorithm}

Algorithm~\ref{alg:robust_ddpg} presents our Action Robust approach adapted to the DDPG algorithm \cite{lillicrap2015continuous}. The action we play during exploration is based on the exploration scheme selected, OU noise adds noise at the action level whereas in parameter space noise we pertube the parameters $\theta$ and $\bar \theta$.

Notice that the critic update is different, in both scenarios, from the default DDPG update rule. The reason is that the critic is updated based on the expectation over the policy, which in the NR-MDP results in the $\alpha$ mixture policy and in the PR-MDP a convex sum of $Q$ values.

\newpage
Figure~\ref{fig:rddpg_diagram} presents a block diagram of our approach for the NR-MDP scenario:

\tikzstyle{block} = [draw, fill=white, rectangle, 
    minimum height=3em, minimum width=6em]
\tikzstyle{sum} = [draw, fill=white, circle, node distance=1cm]
\tikzstyle{pinstyle} = [pin edge={to-,thin,black}]

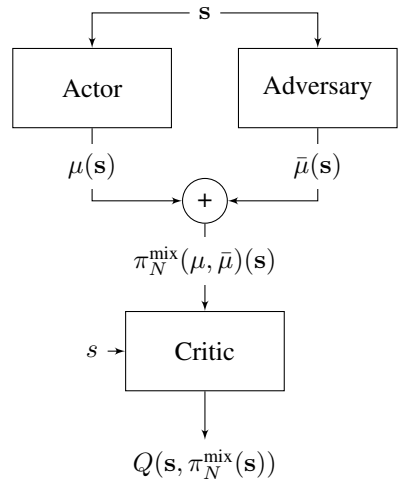
\begin{figure}[H]
\centering
\caption{Action Robust DDPG, NR-MDP}\label{fig:rddpg_diagram}
\begin{tikzpicture}[auto, node distance=2cm,>=latex']
    \node [block] (actor) {Actor};
    \node [sum, below of=actor, node distance=1.5cm, xshift=1.5cm] (sum) {+};
    \node [block, right of=actor, node distance=3cm] (adversary) {Adversary};
    \node [block, below of=actor, node distance=3.5cm, xshift=1.5cm] (critic) {Critic};
    \node [below of=critic, node distance=1.5cm] (out) {$Q(\state, \pi^\text{mix}_N (\state))$};
    \node [left of=critic, node distance=1.5cm] (in_critic) {$s$};
    \node [above of=actor, node distance=1cm, xshift=1.5cm] (in_actor) {$\state$};
    \node [coordinate, above of=actor, node distance=1cm] (input_left) {};
    \node [coordinate, above of=adversary, node distance=1cm] (input_right) {};

    \draw [->] (actor.south)--(actor|-sum.west)node[above, midway, outer sep=5pt, fill=white, yshift=-0.5cm] {$\mu(\state)$}--(sum.west);
    \draw [->] (adversary.south)--(adversary|-sum.east)node[above, midway, outer sep=5pt, fill=white, yshift=-0.5cm] {$\bar \mu(\state)$}--(sum.east);
    \draw [->] (sum)--(critic)node[outer sep=5pt, fill=white, midway, above, yshift=-0.4cm]{$\pi^\text{mix}_N (\mu, \bar \mu)(\state)$};
    \draw [->] (critic)--(out);
    \draw [->] (in_actor.west)--(input_left)--(actor.north);
    \draw [->] (in_actor.east)--(input_right)--(adversary.north);
    \draw [->] (in_critic)--(critic);
\end{tikzpicture}
\end{figure}

We improve the actor (adversary) by taking the gradient of $Q$ w.r.t. $\theta (\bar \theta)$ and performing backpropagation through the critic. Autograd engines \cite{baydin2018automatic} automatically ensure that the gradients propagate directly to the actor (adversary) without affecting the adversary (actor) or the critic. During exploration we simply play $\pi_N^\text{mix}$ a deterministic policy (as it is a convex sum of two deterministic values).

For the PR-MDP the schema is similar to the standard DDPG approach.

\begin{figure}[H]
\centering
\caption{Action Robust DDPG, PR-MDP}\label{fig:pr-mdp rddpg_diagram}
\begin{tikzpicture}[auto, node distance=2cm,>=latex']
    \node [block] (actor) {Actor};
    \node [sum, below of=actor, node distance=1.5cm, xshift=1.5cm] (sum) {+};
    \node [block, right of=actor, node distance=3cm] (adversary) {Adversary};
    \node [above of=actor, node distance=1cm, xshift=1.5cm] (in_actor) {$\state$};
    \node [coordinate, above of=actor, node distance=1cm] (input_left) {};
    \node [coordinate, above of=adversary, node distance=1cm] (input_right) {};

    \draw [->] (actor.south)--(actor|-sum.west)node[above, midway, outer sep=5pt, fill=white, yshift=-0.5cm] {$\mu(\state)$}--(sum.west);
    \draw [->] (adversary.south)--(adversary|-sum.east)node[above, midway, outer sep=5pt, fill=white, yshift=-0.5cm] {$\bar \mu(\state)$}--(sum.east);
    \draw [->] (sum)--(critic)node[outer sep=5pt, fill=white, midway, above, yshift=-0.4cm]{$\pi^\text{mix}_P (\mu, \bar \mu)(\state)$};
    \draw [->] (in_actor.west)--(input_left)--(actor.north);
    \draw [->] (in_actor.east)--(input_right)--(adversary.north);
\end{tikzpicture}
\end{figure}

Figure~\ref{fig:pr-mdp rddpg_diagram} depicts the block diagram during exploration. $\pi^\text{mix}_P$ defines a stochastic policy over $\mu$ and $\bar \mu$. Thus, with probability $1 - \alpha$ we sample action $\mu(\state)$ and otherwise $\bar \mu (\state)$.

\begin{figure}[H]
\centering
\caption{Action Robust DDPG, PR-MDP}\label{fig: pr-mdp critic rddpg_diagram}
\begin{tikzpicture}[auto, node distance=2cm,>=latex']
    \node [block] (actor) {Actor};
    \node [block, below of=actor, node distance=2.5cm] (critic) {Critic};
    \node [below of=critic, node distance=1.5cm] (out) {$Q(\state, \mu (\state))$};
    \node [left of=critic, node distance=1.5cm] (in_critic) {$s$};
    \node [above of=actor, node distance=1cm] (in_actor) {$\state$};
    \node [coordinate, above of=actor, node distance=1cm] (input_left) {};

    \draw [->] (actor.south)--(critic.north)node[above, midway, outer sep=5pt, fill=white, yshift=-0.5cm] {$\mu(\state)$};
    \draw [->] (critic)--(out);
    \draw [->] (input_left)--(actor.north);
    \draw [->] (in_critic)--(critic);
\end{tikzpicture}
\end{figure}

Figure~\ref{fig: pr-mdp critic rddpg_diagram} presents the approach during training. This approach is identical to the standard DDPG approach, except that once taking the gradient $\nabla_\theta Q(s, \mu_\theta (\state))$, we multiply the loss (similar to a change of learning rate) by $1 - \alpha$.

The critic is trained on the expectation over the mixture policies, which in the case of DDPG results in $Q(\state,\action) = r(\state, \action) + \gamma [(1-\alpha) Q(\state', \mu(\state')) + \alpha Q(\state', \bar \mu(\state'))]$.

\newpage
\section{Empirical Results}\label{apndx: empirical results}

\begin{figure}[h]
\centering
\begin{tabular}{>{\centering\arraybackslash}m{.25\linewidth} >{\centering\arraybackslash}m{.25\linewidth} >{\centering\arraybackslash}m{.25\linewidth}}
 No Noise & OU Noise & Param Noise \\ 
 \includegraphics[width=40mm]{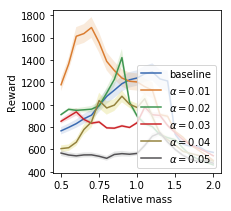} & \includegraphics[width=40mm]{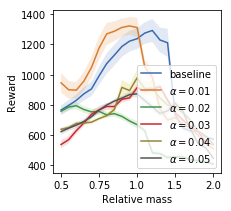} & \includegraphics[width=40mm]{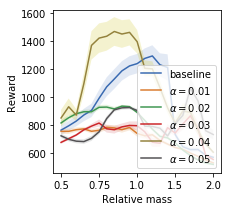} \\
 \includegraphics[width=40mm]{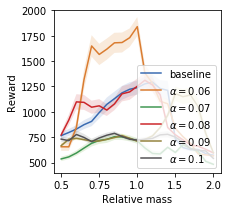} & \includegraphics[width=40mm]{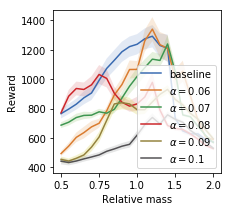} & \includegraphics[width=40mm]{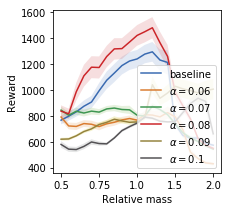} \\
 \includegraphics[width=40mm]{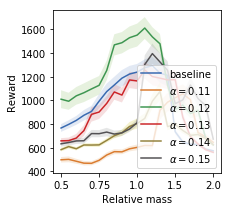} & \includegraphics[width=40mm]{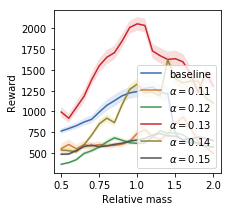} & \includegraphics[width=40mm]{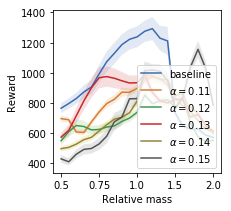} \\
 \includegraphics[width=40mm]{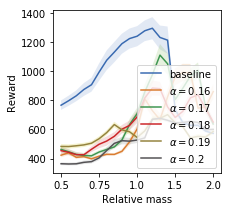} & \includegraphics[width=40mm]{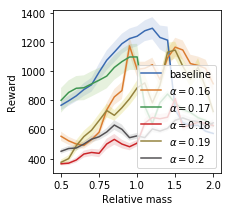} & \includegraphics[width=40mm]{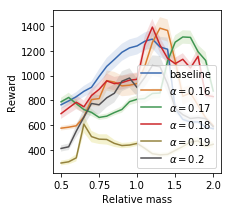}
\end{tabular}
\caption{NR-MDP: exploration and $\alpha$ ablation.}
\end{figure}

\begin{figure}[h!]
\centering
\begin{tabular}{>{\centering\arraybackslash}m{.2\linewidth} >{\centering\arraybackslash}m{.2\linewidth} >{\centering\arraybackslash}m{.2\linewidth} >{\centering\arraybackslash}m{.2\linewidth}}
 \includegraphics[width=40mm]{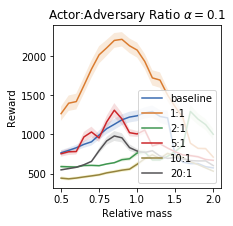} & \includegraphics[width=40mm]{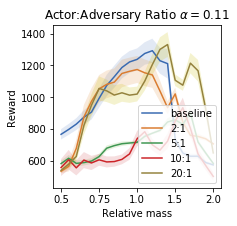} & \includegraphics[width=40mm]{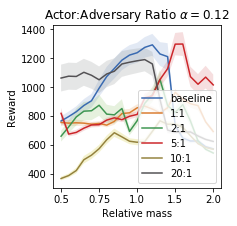} & \includegraphics[width=40mm]{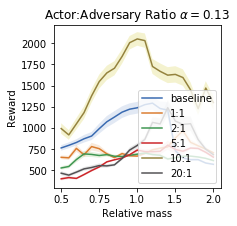}
\end{tabular}
\caption{NR-MDP: $\alpha$ and training ratio ablation.}
\end{figure}

\begin{figure}[h!]
\centering
\begin{tabular}{>{\centering\arraybackslash}m{.25\linewidth} >{\centering\arraybackslash}m{.25\linewidth} >{\centering\arraybackslash}m{.25\linewidth}}
 No Noise & OU Noise & Param Noise \\ 
 \includegraphics[width=40mm]{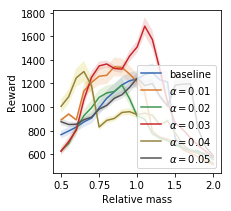} & \includegraphics[width=40mm]{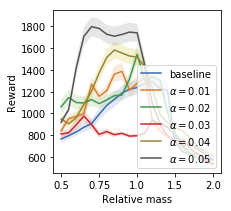} & \includegraphics[width=40mm]{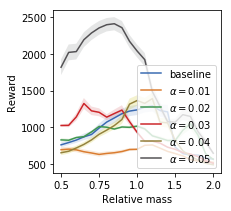} \\
 \includegraphics[width=40mm]{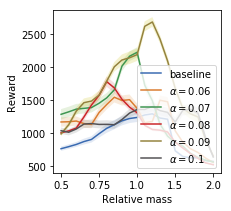} & \includegraphics[width=40mm]{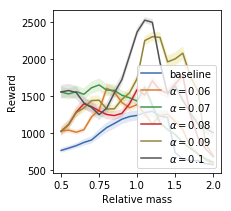} & \includegraphics[width=40mm]{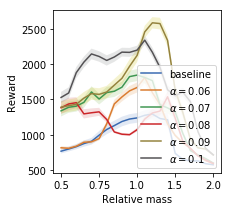} \\
 \includegraphics[width=40mm]{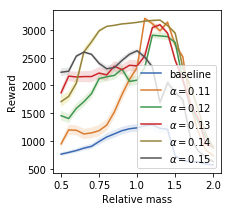} & \includegraphics[width=40mm]{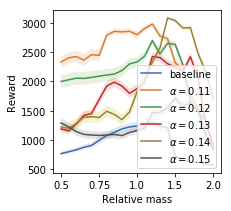} & \includegraphics[width=40mm]{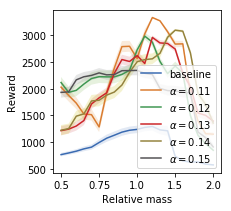} \\
 \includegraphics[width=40mm]{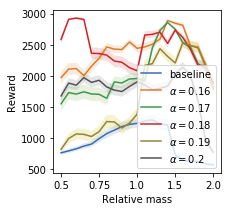} & \includegraphics[width=40mm]{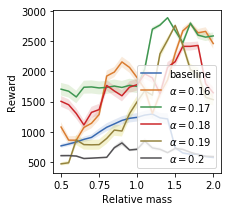} & \includegraphics[width=40mm]{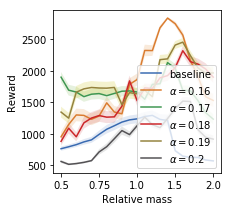}
\end{tabular}
\caption{PR-MDP: exploration and $\alpha$ ablation.}
\end{figure}

\begin{figure}[h!]
\centering
\begin{tabular}{>{\centering\arraybackslash}m{.2\linewidth} >{\centering\arraybackslash}m{.2\linewidth} >{\centering\arraybackslash}m{.2\linewidth} >{\centering\arraybackslash}m{.2\linewidth}}
 \includegraphics[width=40mm]{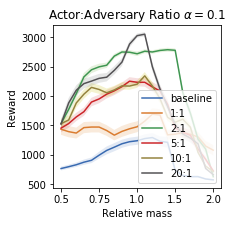} & \includegraphics[width=40mm]{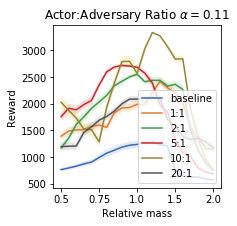} & \includegraphics[width=40mm]{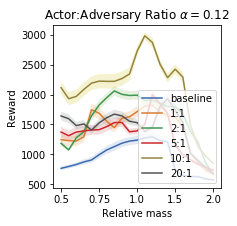} & \includegraphics[width=40mm]{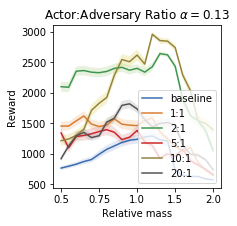}
\end{tabular}
\caption{PR-MDP: $\alpha$ and training ratio ablation.}
\end{figure}

\begin{figure}[h]
\centering
\begin{tabular}{>{\centering\arraybackslash}m{.15\linewidth} >{\centering\arraybackslash}m{.25\linewidth} >{\centering\arraybackslash}m{.25\linewidth} >{\centering\arraybackslash}m{.25\linewidth}}
  & Baseline & NR-MDP & PR-MDP \\ 
 Hopper & \includegraphics[width=40mm]{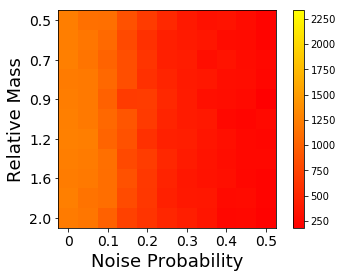} & \includegraphics[width=40mm]{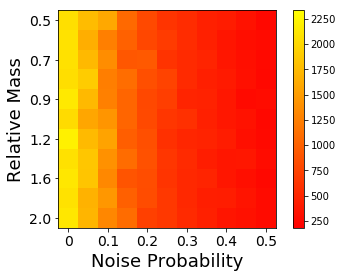} & \includegraphics[width=40mm]{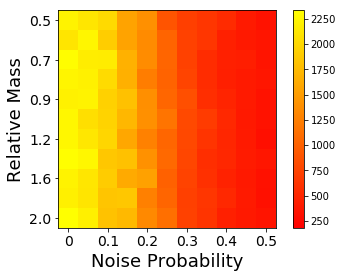} \\
 Walker2d & \includegraphics[width=40mm]{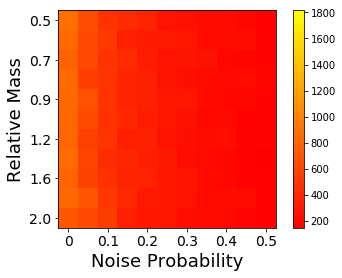} & \includegraphics[width=40mm]{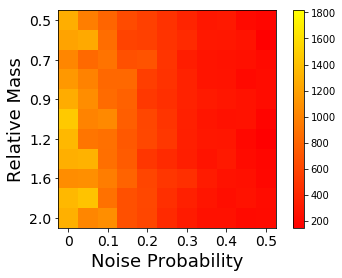} & \includegraphics[width=40mm]{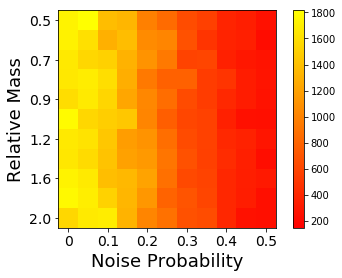} \\
 Humanoid & \includegraphics[width=40mm]{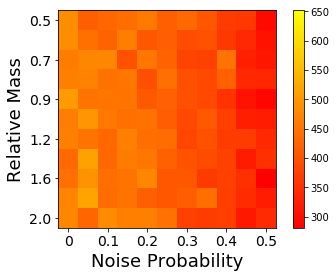} & \includegraphics[width=40mm]{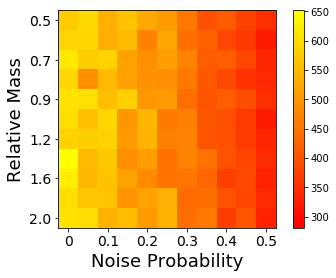} & \includegraphics[width=40mm]{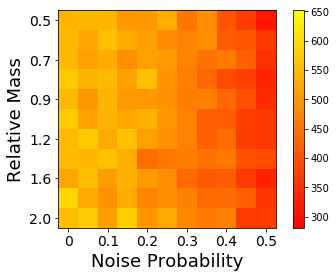} \\
 InvertedPendulum & \includegraphics[width=40mm]{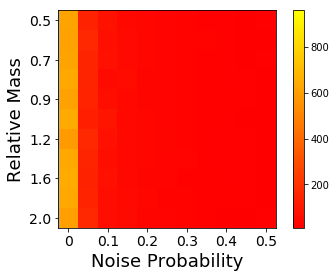} & \includegraphics[width=40mm]{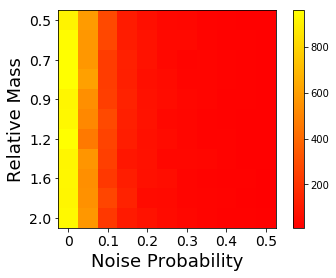} & \includegraphics[width=40mm]{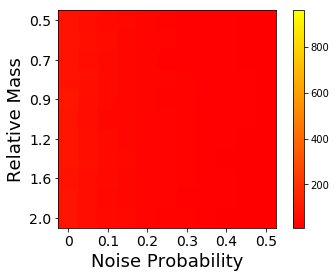}
\end{tabular}
\caption{Robustness to model uncertainty. Noise probability denotes the probability of a randomly sampled noise being played instead of the selected action.}
\end{figure}

\begin{figure}[h!]
\centering
\begin{tabular}{>{\centering\arraybackslash}m{.15\linewidth} >{\centering\arraybackslash}m{.25\linewidth} >{\centering\arraybackslash}m{.25\linewidth} >{\centering\arraybackslash}m{.25\linewidth}}
  & Baseline & NR-MDP & PR-MDP \\
 Swimmer & \includegraphics[width=40mm]{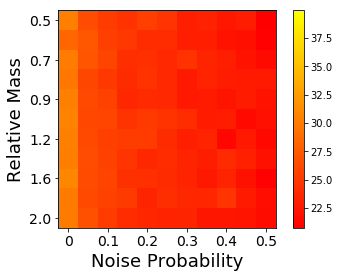} & \includegraphics[width=40mm]{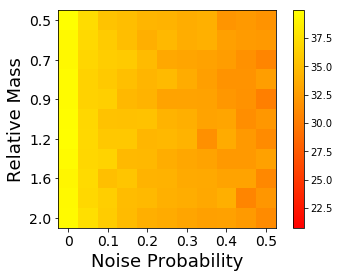} & \includegraphics[width=40mm]{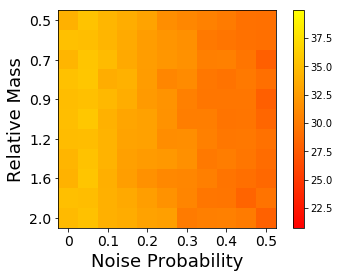} \\
 HalfCheetah & \includegraphics[width=40mm]{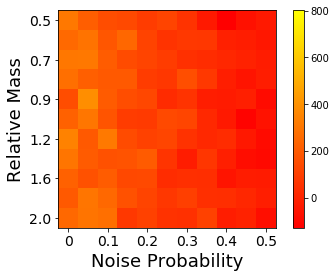} & \includegraphics[width=40mm]{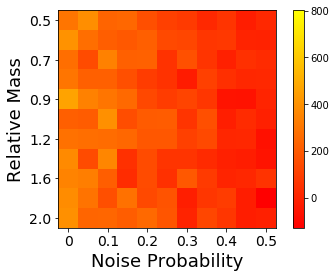} & \includegraphics[width=40mm]{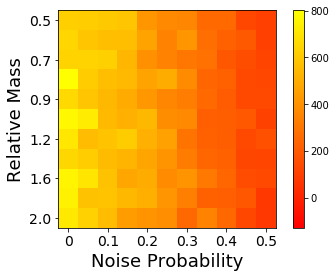} \\
 Ant & \includegraphics[width=40mm]{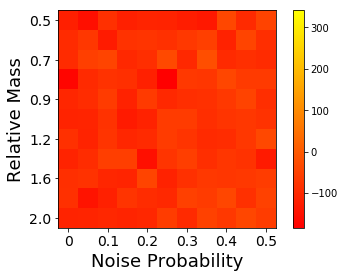} & \includegraphics[width=40mm]{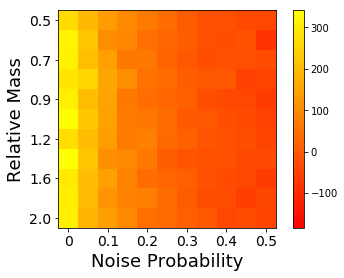} & \includegraphics[width=40mm]{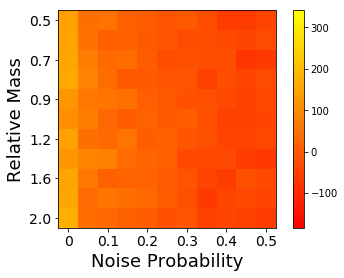}
\end{tabular}
\caption{Robustness to model uncertainty continued. Noise probability denotes the probability of a randomly sampled noise being played instead of the selected action.}
\end{figure}

\begin{figure}[h!]
\centering
\begin{tabular}{>{\centering\arraybackslash}m{.25\linewidth} >{\centering\arraybackslash}m{.25\linewidth} >{\centering\arraybackslash}m{.25\linewidth}}
 \includegraphics[width=40mm]{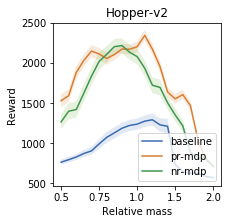} & \includegraphics[width=40mm]{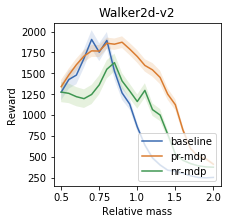} & \includegraphics[width=40mm]{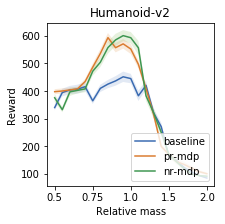} \\
 \includegraphics[width=40mm]{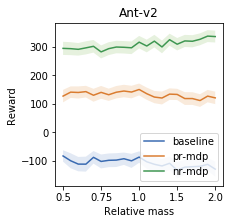} & \includegraphics[width=40mm]{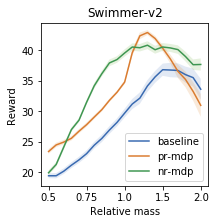} & \includegraphics[width=40mm]{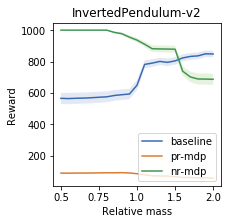} \\
 \includegraphics[width=40mm]{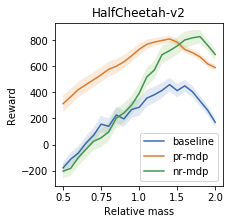} & & \\
\end{tabular}
\caption{Robustness to mass uncertainty.}
\end{figure}


\end{document}